\documentclass{article}
\usepackage[utf8]{inputenc}
\usepackage[margin=1in]{geometry}
\usepackage[numbers,sort]{natbib}
\usepackage{csvsimple,booktabs,adjustbox}
\usepackage[normalem]{ulem}
\usepackage{float}
\usepackage{aux/macros}
\usepackage{pifont}
\usepackage{enumitem}
\newcommand\contributionNote[1]{%
  \begingroup
  \renewcommand\thefootnote{}\footnote{\kern-5pt \textcolor{white}{\rule{5pt}{2ex}}#1}%
  \addtocounter{footnote}{-1}%
  \endgroup
}

\begin{document}
\begin{center}
    \vspace*{0.5cm}

 \LARGE{Mechanism of feature learning in deep fully connected networks and kernel machines that recursively learn features} 
 
\vspace*{0.5cm}

\large{Adityanarayanan Radhakrishnan$^{*, 1, 2}$\\ Daniel Beaglehole$^{*, 4}$ \\ Parthe Pandit$^{3}$ \\ \hspace{.1mm} Mikhail Belkin$^{3, 4}$}

\vspace*{0.5cm}

\normalsize{$^{1}$Massachusetts Institute of Technology.} \\
\normalsize{$^{2}$Broad Institute of MIT and Harvard.} \\
\normalsize{$^{3}$Hal\i c\i o\u glu Data Science Institute, UC San Diego.} \\
\normalsize{$^{4}$Computer Science and Engineering, UC San Diego.} \\
\normalsize{$^{*}$Equal contribution.}
\vspace*{0.5cm}
\end{center}

\setcounter{footnote}{3}
\begin{abstract} 
In recent years neural networks have achieved impressive results on many technological and scientific tasks.  Yet, the mechanism through which these models automatically select \textit{features}, or patterns in data, for prediction remains unclear.  Identifying such a mechanism is key to advancing performance and interpretability of neural networks and promoting reliable adoption of these models in scientific applications.  In this paper, we identify and characterize the mechanism through which deep fully connected neural networks learn features.  We posit the Deep Neural Feature Ansatz, which states that neural feature learning occurs by implementing the average gradient outer product to up-weight features strongly related to model output.  Our ansatz sheds light on various deep learning phenomena including emergence of spurious features and simplicity biases and how pruning networks can increase performance, the ``lottery ticket hypothesis.''  Moreover, the mechanism identified in our work leads to a backpropagation-free method for feature learning with any machine learning model.  To demonstrate the effectiveness of this feature learning mechanism, we use it to enable feature learning in classical, non-feature learning models known as kernel machines and show that the resulting models, which we refer to as \textit{Recursive Feature Machines}, achieve state-of-the-art performance on tabular data. 

\end{abstract}

\section{Introduction}

In the last few years, neural networks have led to major breakthroughs on a variety of applications including image generation~\cite{DALLE}, protein folding~\cite{Alphafold}, and language understanding and generation~\cite{GPT3}.  The ability of these models to automatically learn and utilize problem-specific \textit{features}, or patterns in data, for prediction is thought to be a central contributor to their success~\cite{FeatureLearningEmergenceShi, YangFeatureLearning}.  Thus, a major goal of machine learning research has been to identify the mechanism through which such neural feature learning occurs and which features are selected.  Indeed, understanding this mechanism provides the opportunity to design networks with improved reliability and model transparency needed for various scientific and clinical applications (e.g., natural disaster forecasting, clinical diagnostics). 

Prior works refer to neural feature learning as the change in a network's internal, intermediate representations through the course of training~\cite{YangFeatureLearning, BlakeCengizSelfConsistentDynamical}. Significant research effort~\cite{PreetumLimitations,YangFeatureLearning, YuanzhiLargeLearningRate, DeepLearningTheoryRobertsYaida, FiniteWidthCorrectionsBoris, BaiQuadNTK, NeuralTangentHierarchy, CatapultPhase, LibinQuadratic, ba2022high, 22abbe_staircase, FeatureLearningEmergenceShi, DLSRepresentationReLU, NeuralNetsMultiIndexSGD} has shown the benefits of feature learning in neural networks over non-feature learning models.   Yet, precise characterization of the feature learning mechanism and how features emerge remained an unsolved problem.

In this work, we posit the mechanism for feature learning in deep, nonlinear fully connected neural networks.  Informally, this mechanism corresponds to the approach of progressively re-weighting features in proportion to the  influence they have on the predictions.  Mathematically stated, if $W_i$ denotes the weights of a trained deep network at layer $i$, then Gram matrix $W_i^T W_i$, which we refer to as the \emph{$i$th layer Neural Feature Matrix} (NFM), is proportional to the \ajop of the network with respect to the input to this layer.

As an illustrative example of our results, consider neural networks trained to classify the presence of glasses  in images of faces.  In Fig.~\ref{fig: Overview}A, we compare the NFMs and performance of a non-feature learning network with fixed first layer weights and a feature learning network where all weights are updated.  While the NFM of the non-feature learning model (shown in red) is unchanging through training, the NFM of the feature learning model (shown in green) evolves to represent a pattern corresponding to glasses.  Even though both networks are able to fit the training data equally well, the feature learning model has significantly lower test classification error. A major finding of our work is that we are able to recover the key first layer NFM matrix without access to the internal structure of the neural network.  To illustrate this finding, in Fig.~\ref{fig: Overview}B, we show that the \ajop of a trained neural network with respect to the data is strongly correlated (Pearson correlation greater than $.97$) with the first layer NFM for a variety of classification tasks.

We empirically show that the feature learning mechanism identified in our work unifies previous lines of investigation, which study the relationship between neural feature learning and various aspects of neural networks such as network architecture~\cite{PreetumLimitations} and weight initialization scheme~\cite{YangFeatureLearning}.  In settings where feature learning is argued to occur (e.g.,  in finite width networks and networks initialized near zero), the \ajop is more correlated with the neural feature matrices.  Moreover, our mechanism explains prominent deep learning phenomena including the emergence of spurious features and biases in trained neural networks~\cite{SpuriousFeaturesSoheil, shah2020pitfalls}, grokking~\cite{Grokking}, and how pruning networks can increase performance~\cite{LotteryTicket}.

Importantly, as the \ajop can be computed given any predictor, our result provides a backpropagation-free approach for feature learning with any machine learning model including those models that previously had no feature learning capabilities.  Indeed, we can iterate between training a machine learning model and computing the \ajop of this model to learn features.  We apply this procedure to enable feature learning in class of non-feature learning models known kernel machines~\cite{KernelsBook, AronszajnKernels} and refer to the resulting algorithm as a \textit{Recursive Feature Machine} (RFM).  We demonstrate that RFMs achieve state-of-the-art performance across two tabular data benchmarks covering over $150$ datasets~\cite{FernandezDelgado, TabularDataBenchmark}, thereby highlighting the practical value of leveraging the feature learning mechanism identified in this work.

\begin{figure}[t]
    \centering
    \includegraphics[width=\textwidth]{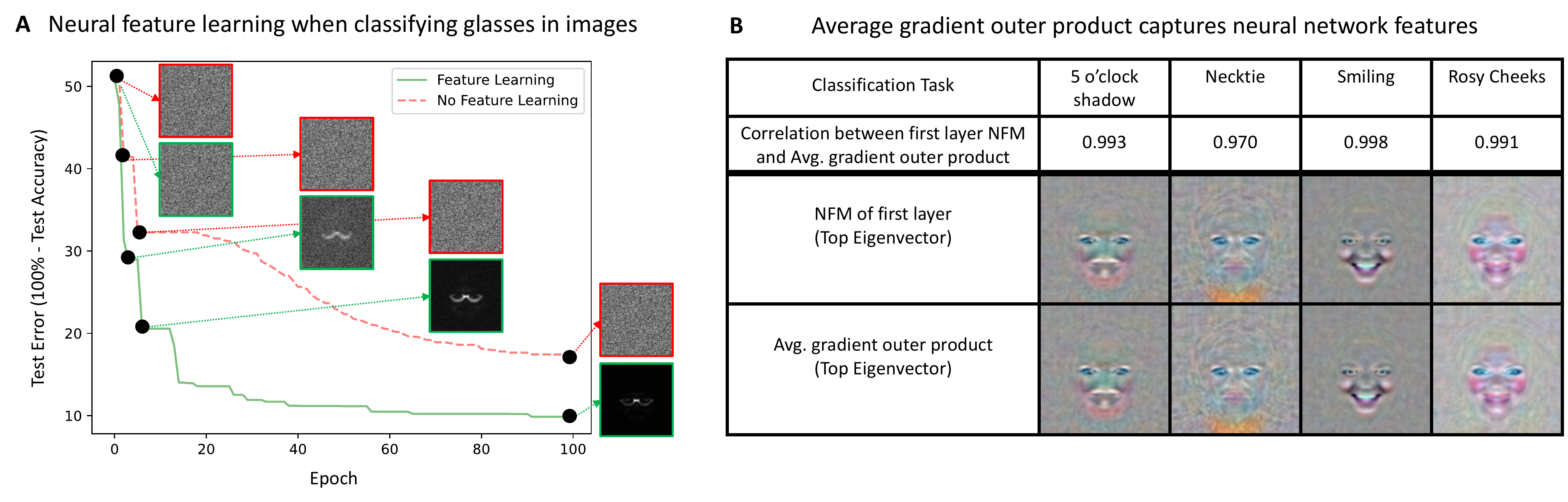}
    \caption{\textbf{(A)} A demonstration of neural feature learning. We train two fully connected neural networks with two-hidden-layers and ReLU activation to classify glasses in image data ($96 \times 96$ images) from the CelebA dataset~\cite{CelebA}, one in which the first layer is not trained and the other in which all layers are trained.  We visualize the diagonal of the first layer neural feature matrix (NFM), $W_1^T W_1$, of the trained networks and observe that the network with better performance learns to select pixels corresponding to glasses. \textbf{(B)}  We show the NFM of a layer is well approximated by the \ajop of the trained network taken with respect to the input of this layer.  The correlation between the first layer NFM and the \ajop with respect to the input data is greater than $0.97$ for four different classification tasks from the CelebA dataset.  We visualize the top eigenvector of these matrices and observe that both are visually indistinguishable and highlight relevant features for prediction.}
    \label{fig: Overview}
\end{figure}

\section{Results} 

Let $f: \mathbb{R}^{d} \to \mathbb{R}$ denote a fully connected network with $L$ hidden layers for $L > 1$, weight matrices $\{W_i\}_{i=1}^{L+1}$, and elementwise activation function $\phi$ of the form 
$$ f(x) = W_{L+1} h_L(x) \qquad ; \qquad h_{\ell}(x) = \phi(W_{\ell-1} h_{\ell-1}(x)) ~\text{for $\ell \in \{2, \ldots, L\}$}$$
with $h_1(x) = x$.  We refer to the terms $h_i(x)$ as the \textit{features} at layer $i$.  We can characterize how features $h_{i+1}(x)$ are constructed by understanding how $W_i$ scales and rotates elements of $h_i(x)$.  These scaling and rotation quantities are recovered mathematically from the eigenvalues and eigenvectors of the matrix $W_i^T W_i$, which is the NFM at layer $i$.  Hence, to characterize how features are updated in any layer of a trained neural network, it suffices to characterize how the corresponding layer's NFM is constructed.  Before mathematically stating how such NFMs are built, we connect NFM construction to the following intuitive procedure for selecting features. 

Given any predictor, a natural approach for identifying important features is to rank them by the magnitude of change in prediction upon perturbation.  When considering infinitesimally small feature perturbations on real-valued predictors, this approach is mathematically equivalent to computing the magnitude of the derivative of the predictor output with respect to each feature.  These magnitudes are computed by the gradient outer product of the predictor given by $(\nabla f(x)) (\nabla f(x))^T$ where $\nabla f(x)$ is the gradient of a predictor, $f$, at a point $x$.\footnote{For predictors with multi-dimensional outputs, we consider the Jacobian Gram matrix given by $(J f(x))^T (J f(x))$, where $J f(x)$ is the Jacobian of a predictor, $f$, at a point $x$.} 

Our main insight, the \emph{Deep Neural Feature Ansatz}, is that deep networks learn features by implementing the above approach for feature selection.  Mathematically stated, we posit that the NFM of any layer of a trained network is proportional to the \ajop of the network taken with respect to the input to this layer.  In particular, let $W_i$ denote the weights of layer $i$ of a deep, nonlinear fully connected neural network, $f$.  Given a sample $x$, let $h_i(x)$ denote the input into layer $i$ of the network, and let $f_i$ denote the sub-network of $f$ operating on $h_i(x)$.  Suppose that $f$ is trained on $n$ samples $\{(x_p, y_p)\}_{p=1}^n$.  Then throughout training, 
\begin{align*}\label{claim:agop_hypothesis}
\hspace{39mm}   W_i^T W_i \propto \frac{1}{n} \sum_{p=1}^{n}  \nabla f_i(h_i(x_p)) \, \nabla f_i(h_i(x_p))^T ~; \tag{Deep Neural Feature Ansatz}
\end{align*}
where $\nabla f_i(h_i(x_p))$ denotes the gradient of $f_i$ with respect to $h_i(x_p)$.\footnote{Additionally, we note that the right hand side of the ansatz can be viewed as a covariance matrix when the gradients are centered.}  We refer to this statement as the Deep Neural Feature Ansatz.  Formally, we prove that the ansatz holds when using gradient descent to layer-wise train (1) ensembles of deep fully connected networks and (2) deep fully connected networks with the trainable layer initialized at zero (see Section~\ref{sec: Theoretical Analysis} and Appendix~\ref{appendix: Theoretical results}).  We note that for the special case of the first layer and for networks with scalar outputs, the right hand side is related to the statistical estimator known as the expected gradient outer product~\cite{xia2002adaptive, RecursiveMultiIndex, EGOPSamory, HardleStokerGradientAveraging}.

\begin{figure}[t]
    \centering
    \includegraphics[width=\textwidth]{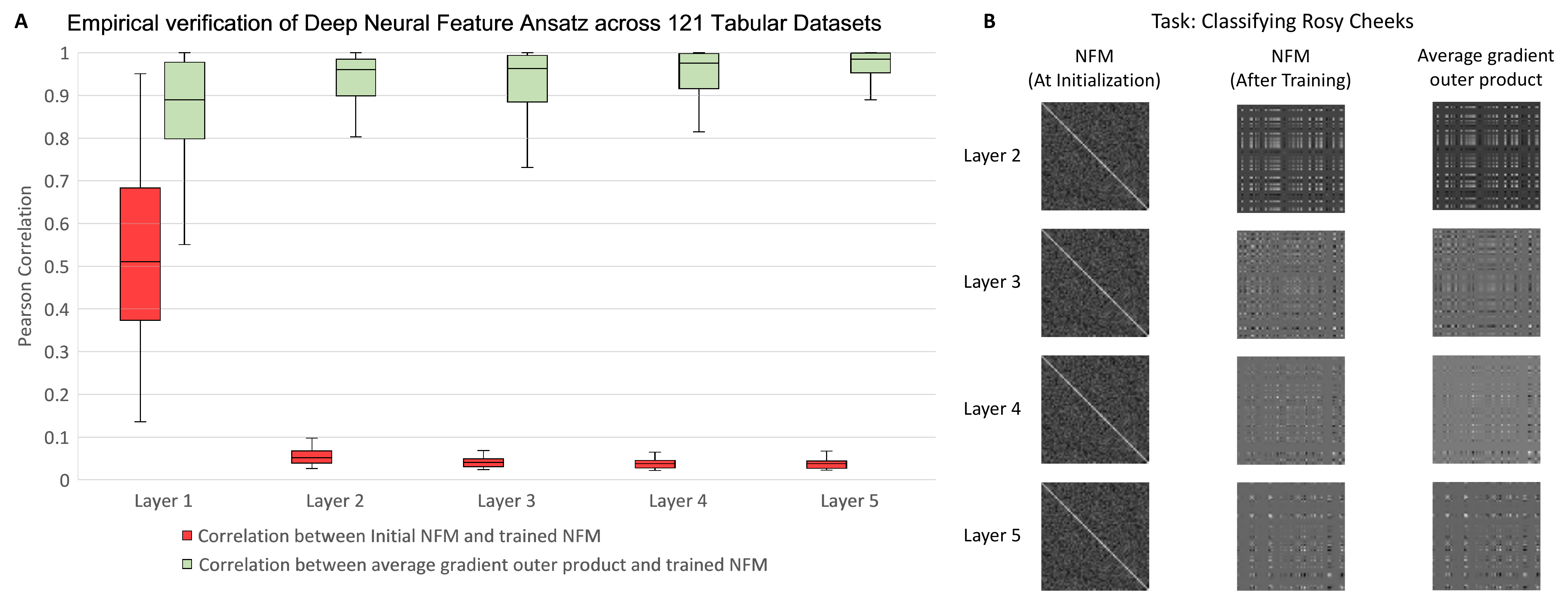}
    \caption{\textbf{(A)} The correlation between the initial NFM and trained NFM (red) and the correlation between \ajop and the trained NFM (green) for each layer in five-hidden-layer ReLU fully connected networks with $1024$ hidden units per layer trained on 121 tabular data classification tasks from~\cite{FernandezDelgado}.  To compute the trained NFM, we subtract the layer weights at initialization from the final weights before computing the Gram matrix.  The higher correlation with the initial NFM in layer 1 is due to the fact that these matrices are smaller than the $1024 \times 1024$ matrices in the remaining layers (on average $28.84$ features across all 121 tasks).  \textbf{(B)} A visualization of the $64 \times 64$ NFM at initialization, the trained NFM, and \ajop for layers $2$ through $5$ of the network trained to classify rosy cheeks from CelebA image data~\cite{CelebA}.  While the NFM at initialization is close to the identity matrix, the NFM after training has a qualitatively different structure that is captured by the \ajop (correlation $>.78$). While we omit the layer $1$ visualization here since the matrices are of size $27648 \times 27648$, the correlation between the first layer NFM after training and the corresponding \ajop is $.99$.}
    \label{fig: Deep Neural Feature Ansatz}
\end{figure}

Next, we empirically validate the ansatz when training all layers of deep fully connected networks across 127 classification tasks.  In particular, in Fig.~\ref{fig: Deep Neural Feature Ansatz}, we train fully connected networks with ReLU activation, five-hidden layers, $1024$ hidden units per layer using stochastic gradient descent on the 121 classification tasks from~\cite{FernandezDelgado}.  In our experiments, we initialize the first layer weights near zero to reduce the impact of the initial weights in computing correlations (see Appendix~\ref{appendix: Methods}).  In Fig.~\ref{fig: Deep Neural Feature Ansatz}A, we observe that the Pearson correlation between the NFMs after training and the \ajops have median value above $.85$ (shown in green) and are consistently higher than the corresponding correlation between the NFMs after training and those at initialization (shown in red).  Note that the gap between the two correlations is larger for layers 2 through 5 since these all have NFMs of dimension $1024 \times 1024$ while the first layer NFM depends on the dimension of the input data, which is on average $28.84$.  In addition to the $121$ tasks, we also validate the ansatz on six different image classification tasks across the CelebA dataset~\cite{CelebA} and Street View House Numbers (SVHN) dataset~\cite{SVHN} (see Appendix Fig.~\ref{appendix fig: Ansatz CelebA SVHN}).  In Fig.~\ref{fig: Deep Neural Feature Ansatz}B, we provide a visualization of the NFMs at initialization, NFMs after training, and \ajops for a fully connected network with five-hidden-layers, $64$ hidden units per layer with ReLU activation trained to classify rosy cheeks in CelebA images.  We observe that while NFMs after training have qualitatively different structure than the NFMs at initialization, such structure is accurately captured by \ajops.  In addition to the above experiments, we empirically validate that the ansatz holds for a variety of commonly used nonlinearities such as leaky ReLU~\cite{LeakyReLU}, hyperbolic tangent, sigmoid, and sinusoid and using standard optimization algorithms such as Adam~\cite{Adam}~(see Appendix Fig.~\ref{appendix fig: Activation}).

Our ansatz unifies several previous lines of investigation into feature learning.  In Appendix Figs.~\ref{appendix fig: Width} and \ref{appendix fig: Init}, we provide empirical evidence that the NFMs and \ajops have greater correlation for finite width networks and networks initialized near zero, which are key regimes in which feature learning is argued to occur~\cite{PreetumLimitations, YangFeatureLearning}.  In particular, we corroborate our results by reporting correlation between the NFMs after training and the \ajops of networks trained on the $121$ tabular classification tasks from~\cite{FernandezDelgado} across $5$ different widths and $5$ initialization schemes.

\subsection{Deep Neural Feature Ansatz sheds light on notable phenomena from deep learning.}  
\label{sec: Deep Learning Phenomena} 

Empirical studies of deep neural networks have brought to light a number of remarkable and often counter-intuitive phenomena.  We proceed to show that the mechanism of feature learning identified by our Deep Neural Feature Ansatz provides an explanation for several notable deep learning phenomena including (1) the emergence of spurious features~\cite{SpuriousFeaturesSoheil, AdversarialExamplesBugsNotFeatures} and simplicity biases~\cite{shah2020pitfalls, huh2021low} ; (2) grokking~\cite{Grokking} ; and (3) lottery tickets in neural networks~\cite{LotteryTicket}.  

\begin{figure}[t]
    \centering
    \includegraphics[width=\textwidth]{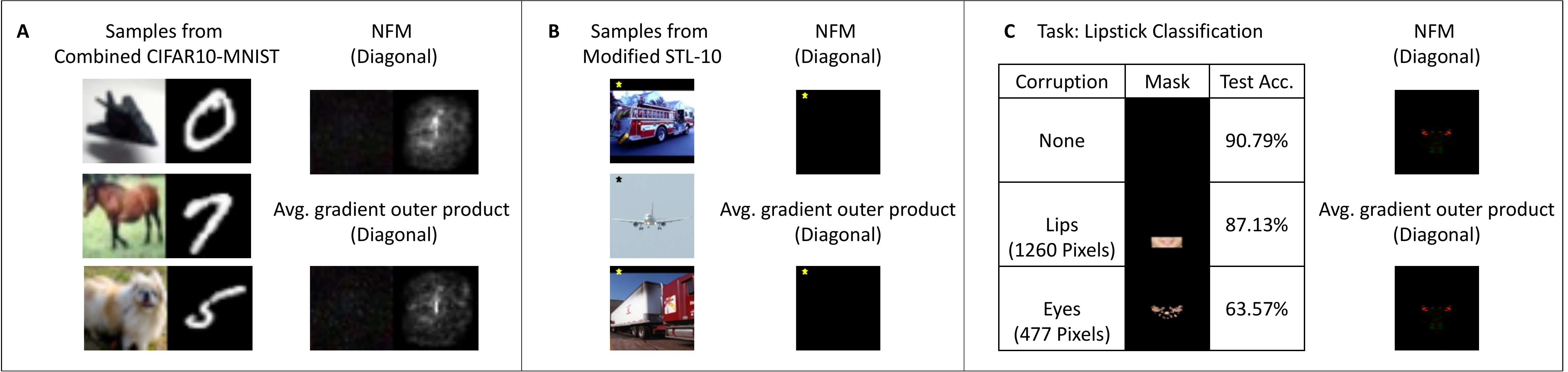}
    \caption{The Deep Neural Feature Ansatz enables identification of and simplicity biases and spurious features in fully connected networks. \textbf{(A)} When trained on $50000$ concatenated $32 \times 64$ resolution images from CIFAR10 and MNIST datasets, the diagonal of the first layer NFM of a five-hidden-layer fully connected ReLU network indicates that digit pixels are primarily used as features for classification.  The \ajop confirms that perturbing these pixels leads to the greatest change in predicted values.  \textbf{(B)} When trained on $1000$ $96 \times 96$ images of planes and trucks modified with a $7 \times 8$ pixel star pattern in the upper left corner, the diagonal of the first layer NFM of a five-hidden-layer ReLU fully connected network indicates the network relies solely on the star pattern for prediction. \textbf{(C)} When trained on $40000$ $96 \times 96$ images from CelebA to classify lipstick, the diagonal of the first layer NFM of a five-hidden-layer ReLU fully connected network indicates the network unexpectedly relies on eye pixels for classification. To corroborate this finding, we find that perturbing test samples by masking the lips leads to only a $3.66\%$ drop in test accuracy, but perturbing the test samples by masking the eyes based on the \ajop leads to a $27.22\%$ drop in test accuracy.}
    \label{fig: Spurious features}
\end{figure}
\paragraph{Spurious features and simplicity biases of neural networks.} The Deep Neural Feature Ansatz implies the emergence of simplicity biases and spurious features in fully connected neural networks.  Simplicity bias refers to the property of neural networks utilizing the ``simplest'' available features for prediction~\cite{shah2020pitfalls, huh2021low, SimplicityBiasHermann, SimplicityBiasPezeshki, SimplicityBiasPreetum} even when multiple features are equally indicative of class labels.  A consequence of simplicity bias is the emergence of spurious features, which are patterns that are correlated but are not necessarily causally related to the predictive targets~\cite{SpuriousFeaturesSoheil, AdversarialExamplesBugsNotFeatures}. 
Examples of neural networks leveraging spurious features include neural networks using the presence of fingers to detect band-aids~\cite{SpuriousFeaturesSoheil} or, problematically, using surgical skin markers to predict malignant skin lesions~\cite{SpuriousFeatureMelanoma}. Frequently,  these spurious features are ``simpler'' than the patterns we consider to be causally predictive. 
 Given their strong correlation with labels, perturbing these simple or spurious features will lead to a larger change in the prediction of a trained model than perturbing other available features, often including those causally related to the predictor.  Hence, the ansatz implies that neural feature learning will reinforce such features.

We demonstrate these phenomena empirically in Fig.~\ref{fig: Spurious features} upon training fully connected networks on three image classification tasks (see Appendix~\ref{appendix: Methods} for training methodology).  In Fig.~\ref{fig: Spurious features}A, we consider the task from~\cite{shah2020pitfalls} and train a model on $50,000$ concatenated images from CIFAR10 and MNIST datasets~\cite{CIFAR10, mnist-lecun1998}.  After training, we visualize the diagonal of the first layer NFM and observe that the model is simply relying on the digit for recognizing the image.  We observe that the \ajop is correlated with the NFM (Pearson correlation $0.8504$), which indicates that perturbing digit pixels leads to the greatest change in prediction.  In Fig.~\ref{fig: Spurious features}B, we show that neural networks will rely primarily on spurious features for prediction even when there are only few such features.  In particular, we trained fully connected networks to classify between $1000$ modified images of trucks and planes from the STL-10 dataset~\cite{STL10} with trucks containing a gold star pattern and planes containing a black star pattern in the upper left corner of the image.  Visualizing the diagonal of the first layer NFM and \ajop indicates that the network simply learns to rely only on the star pattern for prediction.  

Lastly, we showcase the power of our ansatz by using it to identify spurious features for a deep network trained to classify the presence of lipstick in CelebA images.  In Fig.~\ref{fig: Spurious features}C, we observe that the model on original test samples achieves $90.79\%$ accuracy.  Yet, by visualizing the diagonals of the NFM and \ajop, we observe that the trained model is unexpectedly relying on the eyes to determine whether the individual is wearing lipstick.  To further corroborate this finding, we observe that the test accuracy drops only slightly by $3.66\%$ when replacing the lips of all test samples with those of one individual.  If we instead replace the eyes of all test samples according to the mask given by the diagonal of the \ajop, test accuracy drops by $27.22\%$ to slightly above random chance.

\begin{figure}[!t]
    \centering
    \includegraphics[width=\textwidth]{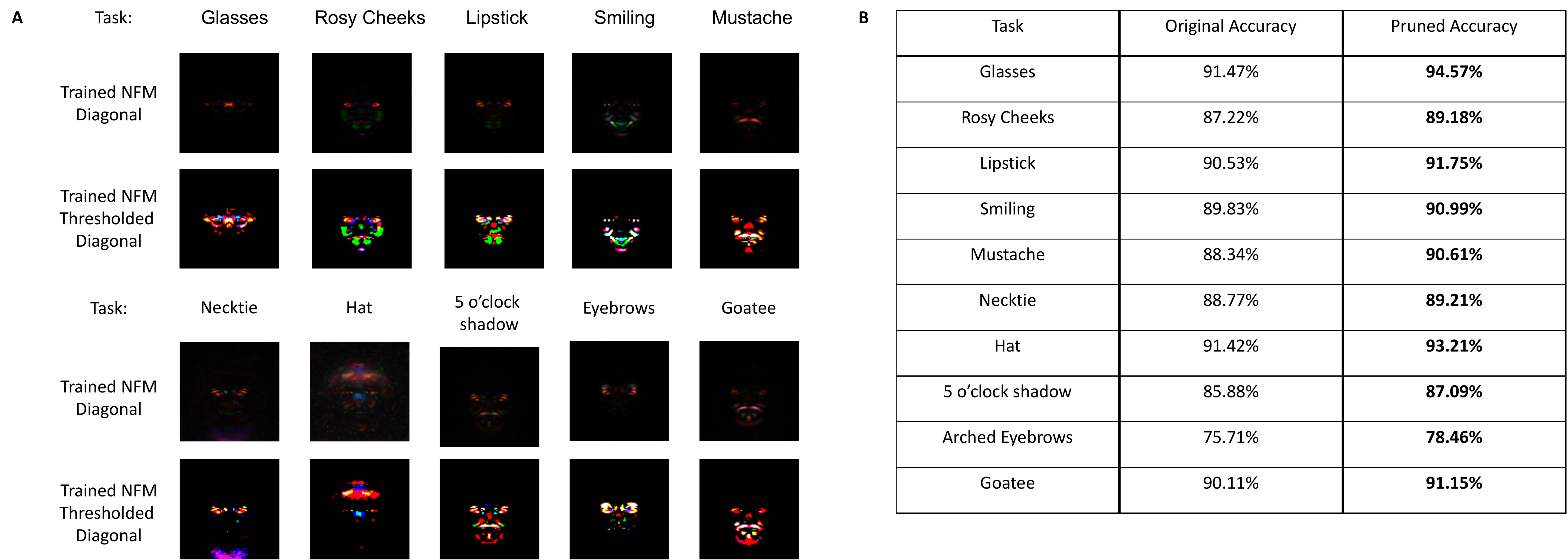}
    \caption{Lottery tickets in fully connected networks. \textbf{(A)} Visualizations of the diagonals of first layer Neural Feature Matrices from a two-hidden-layer, width 1024 ReLU network trained on classification tasks from CelebA and the diagonals after thresholding (replacing with pixel value $1$) the top $2\%$ of pixels.  There are 553 nonzero pixel values in the masked images. \textbf{(B)} Comparison in accuracy after re-training randomly initialized neural networks of the same architecture on the masked images (i.e., pruning $98\%$ of the corresponding columns of the first layer weights).}
    \label{appendix fig: Lottery Tickets}
\end{figure}

\paragraph{{Lottery Tickets.}}  
Introduced in~\cite{LotteryTicket}, the ``lottery ticket hypothesis'' refers to the claim that a randomly-initialized neural network contains a sub-network that can match or outperform the trained network when trained in isolation.  Such sub-networks are typically found by pruning away weights with the smallest magnitude~\cite{LotteryTicket}. The sparsity of feature matrices identified in this work provides direct evidence for this hypothesis.  Indeed, such sparsity is immediately evident when visualizing the diagonals of the feature matrix as in Fig.~\ref{appendix fig: Lottery Tickets}A.  

{In line with the lottery ticket hypothesis, we demonstrate that retraining neural networks after thresholding coordinates of the data corresponding to these sparse regions in the neural feature matrix leads to a consistent increase in performance in many settings.  In Fig.~\ref{appendix fig: Lottery Tickets}A and B, we prune $98\%$ of pixels in CelebA images according to the features identified by neural feature matrix and indeed, observe a consistent increase in predictive performance upon retraining a neural net on the thresholded features.}

\paragraph{{Grokking.}}  {Introduced in recent work~\cite{Grokking}, grokking refers to the phenomenon of deep networks exhibiting a dramatic increase in test accuracy when training past the point where training accuracy is $100\%$. We showcase a similar effect by training neural networks to classify between a subset of $96 \times 96$ resolution images of airplanes and trucks from the STL-10 dataset~\cite{STL10} (training details are presented in Appendix~\ref{appendix: Methods}).    We modify this subset with two key features to enable grokking: (1) the dataset is small with a large class imbalance between the two classes with $500$ examples of airplanes and $53$ examples of trucks and (2) there is a small star of pixels in the upper left corner of each image that is colored white or black based on the class label (see Fig.~\ref{appendix fig: Grokking}A).  The test set is balanced with $800$ examples of each class.}

{In Fig.~\ref{appendix fig: Grokking}B, we observe that grokking aligns with our ansatz.  Indeed, in Fig.~\ref{appendix fig: Grokking}B, we observe that the network can achieve near 100\% training accuracy without any feature learning, but test accuracy remains at roughly 80\%.  Yet, as training continues past this point, the average gradient outer product up-weights pixels corresponding to the star pattern, as indicated by the first layer NFM, and test accuracy improves drastically to $99.38\%$.}

\begin{figure}[!t]
    \centering
    \includegraphics[width=.7\textwidth]{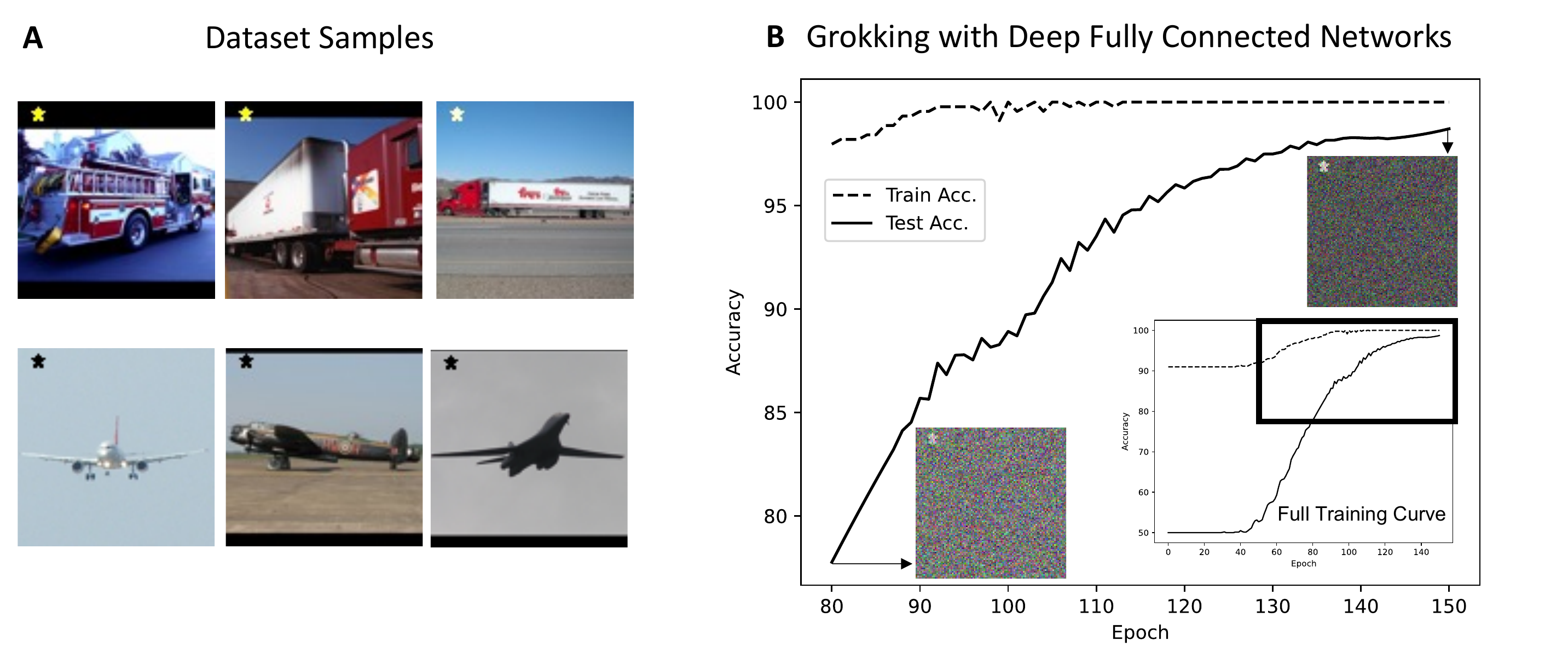}
    \caption{Grokking in fully connected networks. \textbf{(A)} Modified $96 \times 96$ resolution images from a subset of STL-10~\cite{STL10} in which a small star in the upper left indicates whether the image is a truck or an airplane. We use $553$ total training examples with $500$ examples of airplanes and $53$ examples of trucks. \textbf{(B)} A two-hidden-layer fully connected network quickly reaches near $100\%$ training accuracy.  Yet, as training continues past this point, the test accuracy rises drastically from $80\%$ to $99.38\%$.  Corresponding feature matrices (shown as inserts) indicate that test accuracy improved since the network has learned the star pixels that give away the class label.}
    \label{appendix fig: Grokking}
\end{figure}

\subsection{Integrating feature learning into  machine learning models.}

We now  leverage the mechanism of feature learning identified in the ansatz to provide an algorithm for integrating feature learning into any machine learning model.  We then showcase the power of this algorithm by applying it to classical, non-feature learning models known as kernel machines and achieving state-of-the-art performance on tabular datasets.  

A key insight of our ansatz is that neural feature learning occurs through the \ajop, which is a mathematical operation that can be applied to any function.  Given its universality, we can apply it to any machine learning model to enable feature learning.  In particular, we use an iterative two-step strategy that alternates between first training any predictor and then using the \ajop to directly learn features.

To demonstrate the power of this feature learning approach, we apply it to classical, non-feature learning  kernel machines~\cite{KernelsBook}  by (1) estimating a predictor 
using a kernel machine ; (2) learning features using the \ajop of the trained predictor ; and (3) repeating these steps after using the learned features to transform input to the predictor.  For completeness, background on kernels is provided in Appendix~\ref{appendix: kernel ridge regression background}.  Intuitively, training a kernel machine involves solving linear regression after applying a feature transformation on the data.  Unlike traditional kernel functions that are fixed in advance before training, we use kernel functions that incorporate a learnable feature matrix $M$ into the kernel function.  For simplicity, we utilize a generalization of the Laplace kernel given by $K_M(x, z) = \exp(-\gamma \|x - z\|_M)$ where $\gamma > 0$, $M$ is a positive semi-definite, symmetric feature matrix, and $\|x - z\|_M^2 := (x - z)^T M (x-z)$ denotes the Mahanolobis distance between data points $x, z$.\footnote{We note that in statistical literature this distance is defined by $d_M(x, z) = \sqrt{(x -z)^T M^{-1} (x-z)}$~\cite{MahanolobisDistance}, but here, we make use of the notation from metric learning literature~\cite{MetricLearningBook}, which omits the inverse.  We additionally note that Mahanolobis kernels can be extended to general, non-radial kernels by considering kernels of the form $K_M(x, z) = K(M^{\frac{1}{2}} x, M^{\frac{1}{2}} z)$.}  We now alternate between using kernel regression with the kernel function, $K_M$, to estimate a predictor and using the \ajop to update the feature matrix, $M$.  We refer to the resulting algorithm, presented in Algorithm~\ref{alg:RFM}, as a \textit{Recursive Feature Machine} (RFM).

\begin{algorithm}[!ht]
\caption{Recursive Feature Machine (RFM)}\label{alg:RFM}
\begin{algorithmic}
\Require $X, y, K_M, T$ \Comment{Training data: $(X, y)$, kernel function: $K_M$, and number of iterations: $T$}
\Ensure $\alpha, M$ \Comment{Solution to kernel regression: $\alpha$, and feature matrix: $M$}
\State $M = I_{d \times d}$ \Comment{Initialize $M$ to be the identity matrix}
\For{$t \in T$}
    \State $K_{train} = K_M(X, X)$\Comment{$K_M(X, X)_{i,j} := K_M(x_i, x_j)$} 
    \State $\alpha = yK_{train}^{-1}$
    \State $M = \frac{1}{n} \sum_{x \in X} (\nabla f(x)) (\nabla f(x))^T$\Comment{$f(x) = \alpha K_M(X, x)$ with $K_M(X, x)_{i} := K_M(x_i, x)$} 
\EndFor
\end{algorithmic}
\end{algorithm}

\begin{wrapfigure}{r}{.57\textwidth}
    \centering
    \includegraphics[width=.4\textwidth]{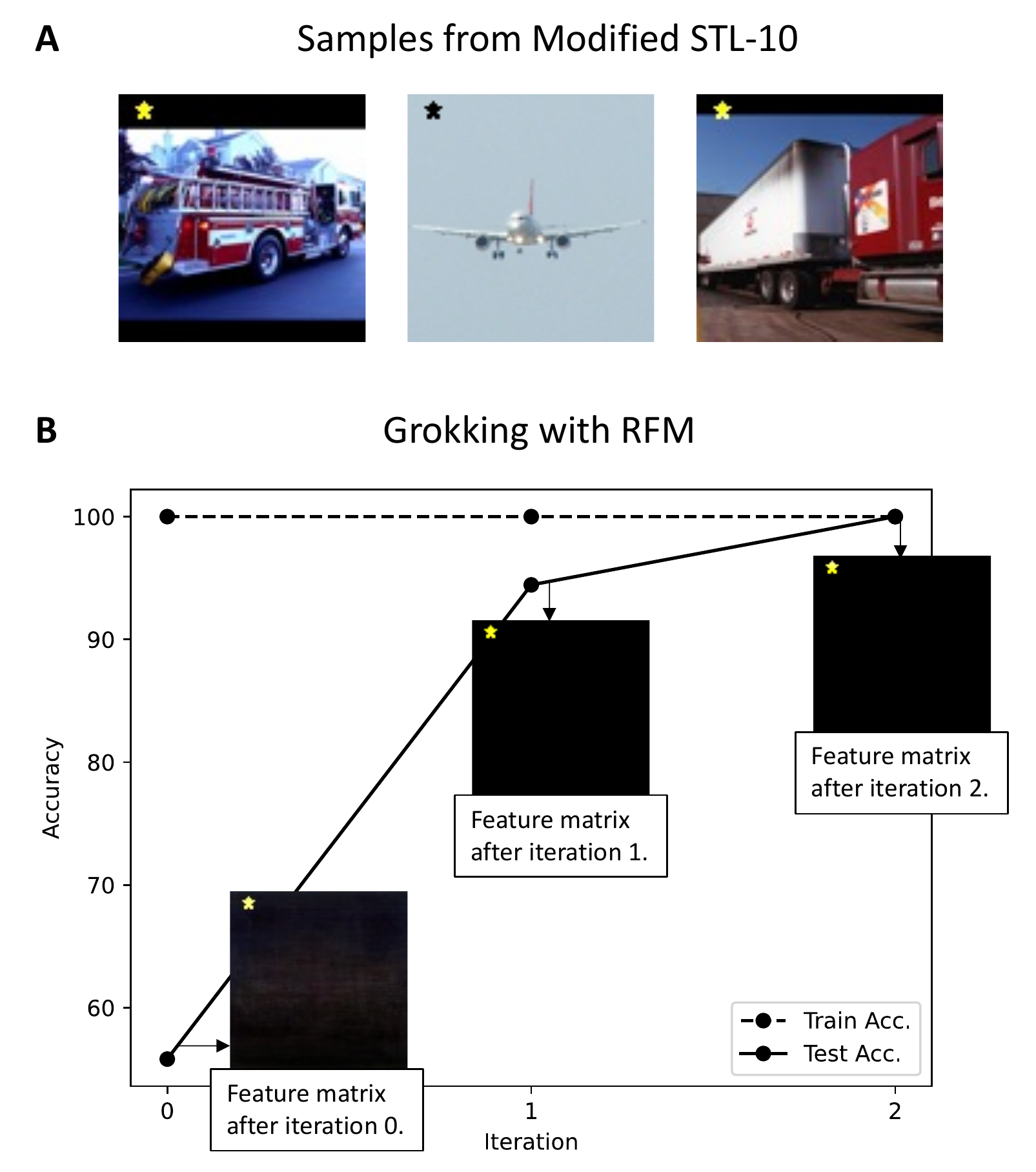}
    \caption{Grokking in RFMs trained on modified STL-10.  \textbf{(A)} Samples from modified STL-10 (see Section~\ref{sec: Deep Learning Phenomena} and Fig.~\ref{appendix fig: Grokking} for dataset details).  \textbf{(B)} While training accuracy is always $100\%$, iteration leads to a drastic rise in test accuracy from $55.8\%$ to $100\%$.  Feature matrices (shown as inserts) indicate that test accuracy improved since the RFM has learned the star pixels that give away the class label. }
    \label{fig: RFM Grokking}
\end{wrapfigure}

In Appendix~\ref{appendix: Methods} and Appendix Figs.~\ref{appendix fig: NN and RFM features CelebA} and~\ref{appendix fig: NN and RFM features SVHN Low Rank}, we compare features learned by RFMs and deep fully connected networks and demonstrate remarkable similarity between RFM features and first layer features of deep fully connected neural networks.  We show that the correlation between the top eigenvector of the first layer NFM after training and that of the RFM feature matrix, $M$, is consistently greater than $.99$ for $12$ different classification tasks from CelebA.  We also show high correlation between RFM features and first layer NFM features for SVHN and low rank polynomial regression tasks from~\cite{PreetumLimitations} and \cite{DLSRepresentationReLU}.  Lastly, in Appendix \ref{appendix: Kernel Alignment}, we discuss connections between RFMs and prior literature on kernel alignment~\cite{cristianini2001kernel,wang2015overview}.  

Given that RFMs use the same feature learning mechanism as neural networks, these models exhibit the deep learning phenomena discussed earlier, i.e., grokking, lottery tickets, and simplicity biases.  In Fig.~\ref{fig: RFM Grokking}, we showcase that RFMs perform grokking on the same dataset used in Section~\ref{sec: Deep Learning Phenomena} and Fig.~\ref{appendix fig: Grokking}.  We show that RFMs exhibit lottery ticket and simplicity bias phenomena in Appendix Figs.~\ref{appendix fig: RFM Lottery Ticket} and ~\ref{appendix fig: RFM Simplicity Bias}.

\subsection{Recursive Feature Machines provide state-of-the-art results on tabular data.}  We demonstrate the immediate practical value of the integrated feature learning mechanism by demonstrating that RFMs achieve state-of-the-art results on two tabular benchmarks containing over $150$ datasets.  The first benchmark we consider is from~\cite{FernandezDelgado}, which compares the performance of $179$ different machine learning methods including neural networks, tree-based models, and kernel machines on $121$ tabular classification tasks.  In Fig.~\ref{fig: 121 Datasets}A, we show RFMs outperform these classification methods, kernel machines using the Laplace kernel, and kernel machines using the recently introduced Neural Tangent Kernel (NTK)~\cite{NTKJacot} across the following commonly used performance metrics:

\begin{itemize}
    \item Average accuracy: The average accuracy of the classifier across all datasets.
    \item P90/P95: The percentage of datasets on which the classifier obtained accuracy within 90\%/95\% of that of the best performing model.  
    \item PMA: The percentage of the maximum accuracy achieved by a classifier averaged across all datasets. 
    \item Friedman rank: The average rank of the classifier across all datasets.      
\end{itemize}

\begin{figure}[t]
    \centering
    \includegraphics[width=\textwidth]{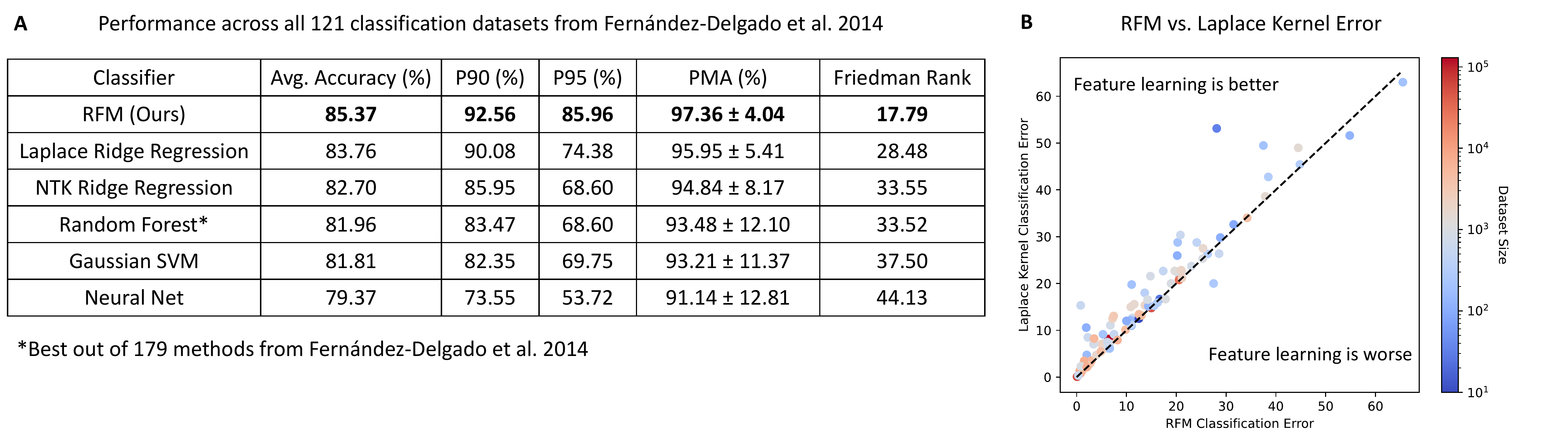}
    \caption{\textbf{(A)} Comparison of $182$ models including RFMs, NTKs, random forests, and fully connected neural networks on $121$ tabular datasets from~\cite{FernandezDelgado}.  All metrics, models, and training details are outlined in Appendix~\ref{appendix: Methods}.  RFMs took 40 minutes to achieve these results while Neural Nets took 5 hours (both measurements are in wall time on a server with two Titan Xp GPUs). \textbf{(B)} To determine the benefit of feature learning, we compare the error rate (100\% - accuracy) between RFMs and Laplace kernels, which are equivalent to RFMs without feature learning.}
    \label{fig: 121 Datasets}
\end{figure}

We note that while some of the datasets  contain up to $130,000$ training examples, RFMs are computationally fast to train through the use of pre-conditioned linear system solvers such as EigenPro~\cite{EigenPro, EigenProGPU}.  Indeed, RFMs take 40 minutes  to achieve these results while neural networks took 5 hours (both measurements are in wall time on a server with two Titan Xp GPUs).  In Fig.~\ref{fig: 121 Datasets}B, we analyze the benefit of feature learning by comparing the difference in error (100\% - accuracy) between RFMs and the classical Laplace kernel, which is equivalent to an RFM without feature learning.  We observe that the Laplace kernel generally results in higher error than the RFM for larger datasets.

In Appendix Fig.~\ref{fig: Transformers} and Tables~\ref{appendix: table reg without cat},~\ref{appendix: table class without cat},~\ref{appendix: table reg with cat}, and~\ref{appendix: table class with cat} we additionally compare RFMs to two transformer models~\cite{FTTransformer, SAINTTransformer}, ResNet~\cite{ResNet}, and two gradient boosting tree models~\cite{scikit-learn, XGBoost} across regression and classification tasks on a second tabular benchmark~\cite{TabularDataBenchmark}. Consistent with our findings on the first tabular benchmark, we observe that RFMs generally outperform tree-based models and neural networks at a fraction of the computational cost ($3600$ compute hours for RFMs while all other methods are with $20,000$ compute hours).

\subsection{Theoretical Evidence for Deep Neural Feature Ansatz}
\label{sec: Theoretical Analysis} 

We now present theoretical evidence for the Deep Neural Feature Ansatz. A summary of all theoretical results for the Deep Neural Feature Ansatz is presented in Table~\ref{tab:models}.  

\begin{table}[h!]
\centering
\begin{tabular}{cccccccc}
    \toprule
    Result & Activation & Steps & Depth & Outer layers & Initialization & GIA & \# Samples\\
    \midrule    
    Proposition~\ref{prop: 1 example} & Any & Any & Any & Fixed & Zero & No & 1 \\
    Proposition~\ref{prop: 1 step, multiple examples} & Any & 1 & Any & Fixed & Zero & No & Any \\
    Proposition~\ref{prop:linear} & Linear & Any & $2$ & Fixed, i.i.d. & Zero & No & Any\\  
    Proposition~\ref{prop:layerbothlayers} & Linear & $2$ & $2$ & Trainable, i.i.d & Zero & No & Any\\        
    \Cref{thm:nonlinear_resampled} & ReLU & Any & Any & Fixed, i.i.d. & Any & Yes & Any\\
    \bottomrule
\end{tabular}
\caption{Settings for which we prove the Deep Neural Feature Ansatz. \textit{Activation} refers the type of network activation function. \textit{Steps} refers to the number of steps of gradient descent for which the proof holds. \textit{Depth} refers to the depth of the neural network considered. \textit{Outer layers} describes how the layers other than the first are initialized and trained. \textit{Initialization} refers to the initialization method of the first layer weights. \textit{GIA} indicates whether the gradient independence ansatz is required for the result to hold. \textit{\# Samples} denotes the number of training samples considered.}
\label{tab:models}
\end{table}

To provide intuition as to when the ansatz holds, we first analyze the general setting in which we train models, $f: \mathbb{R}^{d} \to \mathbb{R}$, of the form $f(x) = g(Bx)$ with $g: \mathbb{R}^{k} \to \mathbb{R}$ by updating the weight matrix $B$ using gradient descent.  This setting encompasses any type of neural network in which the first layer is fully connected and the only trainable layer.  We begin with Propostion~\ref{prop: 1 example} below (proof in Appendix~\ref{appendix: Theoretical results}), which establishes the ansatz for such functions trained on one training example $(x, y) \in \mathbb{R}^{d} \times \mathbb{R}$.

\begin{prop}
\label{prop: 1 example}
Let $f(z) = g(Bz)$ with $f: \mathbb{R}^{d} \to \mathbb{R}$ and $g: \mathbb{R}^{k} \to \mathbb{R}$.  Given one training sample $(x, y)$, suppose that $f$ is trained to minimize $\frac{1}{2}(y - f(x))^2$ using gradient descent.  Let $B^{(\ell)}$ denote $B$ after $\ell$ steps of gradient descent.  If $B^{(0)} = \mathbf{0}$ and $f_t(z) := g(B^{(t)}z)$, then for all time steps $t$: 
\begin{align*}
    \nabla f_t(z) \nabla f_t(z)^T \propto {B^{(t)}}^T B^{(t)}~.
\end{align*}
\end{prop}

Informally, this example demonstrates that feature learning is maximized when the ansatz holds.  Namely, if initialization is nonzero, then ${B^{(t)}}^T B^{(t)}$ contains a term corresponding to initialization ${B^{(0)}}^T B^{(0)}$, which disrupts exact proportionality in the ansatz but also affects the quality of features learned.  In the extreme case of non-feature learning models such as neural network Gaussian processes~\cite{NealNNGP} where the first layer weights are drawn from a standard Gaussian distribution and fixed, then ${B^{(t)}}^T B^{(t)}$ has rank almost surely $\min(k, d)$ while $\nabla f_t(z) \nabla f_t(z)^T$ is a rank 1 matrix.  From this perspective, we argue that the ansatz provides a precise characterization of feature learning. 

The difficulty in generalizing Proposition~\ref{prop: 1 example} to multiple steps is that the terms $\nabla g(B^{(t)}x_i)$ are no longer necessarily equal for all examples after 1 step.  Thus, instead of proving the result for this general class of functions, we instead turn to classes of functions corresponding to fully connected networks.  In particular, Theorem~\ref{thm:nonlinear_resampled} (proof in Appendix~\ref{appendix: Theoretical results}) establishes the ansatz in the more general setting of deep, nonlinear fully connected networks.  Before stating this theorem, we introduce the relevant notation for deep neural networks and the gradient independence ansatz used in our proof. 
\paragraph{Notation.} 
Let $f: \mathbb{R}^{d} \to \mathbb{R}$ denote an $L$ hidden layer network with element-wise activation $\phi: \mathbb{R} \to \mathbb{R}$ of the form:
\begin{align}\label{eq:g_def}
    g(x) = a \tran \phi_{(L)}(x)  \qquad  \qquad \phi_{(\ell)}(x) =     \begin{cases}
        \phi\round{\sqrt{\frac{c_\phi}{k_{\ell}}} W_{\ell} \phi_{(\ell-1)}(x)} & \ell \in \{1, \ldots, L\} \quad\\
        x & \ell = 0
    \end{cases}  
\end{align}
where $a \in \mathbb{R}^{k_{L}}$ and $ W_\ell \in \Real^{k_\ell \times k_{\ell-1}}$ for $\ell \in [L]$ with $k_\ell \in \mathbb{Z}_+$ and $k_0 = d$.  We denote row $k$ of weight matrix $W_\ell$ by $W_{\ell,k} \in \Real^{d \times 1}$.

\vspace{3mm}

\noindent \textbf{Gradient Independence Ansatz (GIA).\footnote{This is often called an assumption in the literature (e.g., Assumption 2.2 in \cite{yang2019scaling}). We prefer the word ansatz following the terminology from~\cite{BlakeCengizSelfConsistentDynamical}, as this is a simplifying principle rather than a true mathematical assumption.}}
\label{assum:grad_independence} \textit{In computing gradients, whenever we multiply by a weight matrix, $W_{\ell}$, we can instead multiply by an i.i.d. copy of $W_{\ell}$ without changing the gradient.}

\vspace{3mm}

We note that the GIA has been used in a range of works on analyzing neural networks including~\cite{yang2019scaling,pennington2017resurrecting,yang2017mean,chen2018dynamical,xiao2018dynamical,yang2018deep,schoenholz2016deep} and is implicit in the original NTK derivation~\cite{NTKJacot}.\footnote{This condition is required for establishing the closed form of the deep NTK presented in~\cite{NTKJacot} as observed in~\cite{yang2019scaling} but is not needed to establish transition to linearity (e.g., \cite{BelkinNTKLinear}).} In~\cite{CNTKArora}, the authors rigorously prove the gradient independence ansatz for fully connected neural networks with ReLU  activation functions.

\begin{theorem}
\label{thm:nonlinear_resampled}
Let $f$ denote an $L$-hidden layer network with ReLU activation ${\phi(z) = \max\{0,z\}}$. Suppose we sample weights $a_{k'}, W_{\ell}$ for $\ell > 1$ in an i.i.d. manner so that $\mathbb{E}\sbrac{a_{k'}^2} = 1$, $\mathbb{E}\sbrac{W_{\ell,k''}^2} = 1$, $\mathbb{E}[a_{k'}]=0$, and $\mathbb{E}\sbrac{W_{\ell,k''}} = 0$. Suppose $W_1$ is fixed and arbitrary.  Let $\{(x_i, y_i)\}_{i=1}^{n} \subset \mathbb{R}^{d} \times \mathbb{R}$.  If $x \sim \mathcal{N}(0,I_d)$, then
\begin{align*}
    \frac{1}{k_1} W_{1}\tran W_{1} &= \mathbb{E}_{x} \sbrac{\lim_{k_2,\ldots,k_L \to \infty} \mathbb{E}_{a} \left[ \nabla_x f(x) {\nabla_x f(x)}\tran \right]  }~.
\end{align*}
\end{theorem}

Theorem~\ref{thm:nonlinear_resampled} can be directly applied in a recursive fashion to prove the ansatz holds for any layer of a deep network.  In particular, we simply consider the layer-wise training scheme in which we show the ansatz holds for the first layer, and then fix the first layer and apply Theorem~\ref{thm:nonlinear_resampled} to the second layer and proceed inductively.  
\section{Summary and Discussion}

\paragraph{Summary of results.} Characterizing the mechanism of neural feature learning has been an unresolved problem that is key to advancing performance and interpretability of neural networks.  In this work we posited the Deep Neural Feature Ansatz, which stated that neural feature learning occurs by up-weighting the features most influential on model output, a process that is formulated mathematically in terms of the \ajop.  Our ansatz unified previous lines of investigation into neural feature learning and explained various deep learning phenomena.  An important insight from our ansatz was that the \ajop could instead be used to learn features with any machine learning model.  We showcased the power of this insight by using it to enable feature learning with classical, non-feature learning models known as kernel machines.  The resulting algorithm, which we referred to as Recursive Feature Machines, achieved state-of-the-art performance on tabular benchmarks containing over $150$ classification and regression tasks.

\paragraph{Connections and implications.}  We conclude with a discussion of connections between our results and machine learning literature as well as implications of our results. 

\vspace{3mm} 

\noindent \textit{Advancing interpretability in deep learning.}  A key area of practical interest is understanding and interpreting how neural networks make predictions.  There is a rich literature on gradient-based methods for understanding features used by deep networks for image classification tasks~\cite{ZeilerVisualizing, GradCAM, DeepLift}.  These methods utilize gradients of trained networks to identify patterns that are important for prediction in a single data point. Rather than focus on features relevant for individual samples, our ansatz directly provides a characterization of Neural Feature Matrices, which capture features the network selects across all data points.  We demonstrated how our ansatz shed light on the emergence of spurious features in neural networks and how it could be leveraged to identify such spurious features.  We envision that the transparency provided by our ansatz can serve as a key tool for increasing interpretability and mitigating biases of neural networks more generally. 

\vspace{3mm}

\noindent \textit{Building state-of-the-art models at a fraction of the cost by streamlining feature learning.}  Neural networks simultaneously learn a predictor and features through backpropagation.  While such simultaneous learning is a remarkable aspect of neural networks, our ansatz shows that it can also lead to inefficiencies during training.  For example, in the initial steps of training, the features are selected based on a partially trained predictor, and the resulting features can be uninformative.  To streamline the feature learning process, we showed that we can instead train a predictor and then estimate features directly via the \ajop. This approach has already led to an improvement in performance and a reduction in time in our experiments, specifically from 5 hours for training neural networks to 40 minutes for RFM across $121$ tabular data tasks from~\cite{FernandezDelgado}.  A natural next direction is to extend this direct feature learning approach for fully connected networks to streamline training of general network architectures including convolutional networks, graph neural networks, and transformers.  We envision that using the \ajop as an alternative to backpropagation could reduce the sizeable training costs associated with state-of-the-art models, including large language models for which fully connected networks form the backbone.

\vspace{3mm} 

\noindent  \textit{The role of width: Transition to linearity vs. feature learning}.  Under the NTK initialization scheme, very wide neural networks undergo a transition to linearity and implement kernel regression with a kernel function that is not data adaptive and depends entirely on the network architecture~\cite{NTKJacot, BelkinNTKLinear}.  On the other hand, more narrow  neural networks simultaneously learn both a predictor and features.  Thus, network width  modulates between two different regimes: one in which networks implement non-data-adaptive kernel predictors and another in which networks learn features.  While this remarkable property highlights the flexibility of deep networks, it also illustrates their complexity.  Indeed, simply increasing width under a particular initialization scheme increases the representational power of a neural network while decreasing its ability to learn features.  In contrast, by separating predictor learning and feature learning into separate subroutines, we can circumvent such modelling complexity without sacrificing performance. 

\vspace{3mm}
\noindent {\textit{The role of depth.}  Our Deep Neural Feature Ansatz provided a way to capture features at deeper layers of a fully connected network by using the average gradient outer product.  Yet, our implementation of RFMs leveraged this feature learning mechanism to capture only features of the input, which  corresponded to the first layer features learned by fully connected networks. Interestingly, despite only using first layer features, RFMs provided state-of-the-art results on tabular datasets, matching or outperforming deep fully connected networks across a variety of tasks.  Thus, an interesting direction of future work is to understand the exact nature of deep feature learning and to characterize the architectures, datasets, and settings for which deep feature learning is beneficial.}

\vspace{3mm}

\noindent \textit{Empirical NTK}. Recently, a line of works have studied the connection between kernel learning and neural networks through the time-dependent evolution of the NTK~\cite{fort2020deep}. An insightful work~\cite{long2021properties} showed that  the \textit{after kernel}, i.e. the empirical NTK at the end of training, matches the performance of the original network.   Interestingly,~\cite{atanasov2021neural} highlighted another benefit of empricial NTK by showing that as a neural network is trained, the empirical NTK increases alignment with the ideal kernel matrix.  Given the similarity in features learned between RFMs and neural networks, we believe that RFMs may be an effective means of approximating the after kernel without training neural networks. 

\vspace{3mm}

\noindent \textit{Connections to other  statistical and machine learning methods.}  Our result connects neural feature learning to a number of classical methods from statistics and machine learning, which we discuss below.  
\begin{itemize}[itemindent=0pt,labelsep=1ex,leftmargin=4ex]

\item \textit{Supervised dimension reduction.} 
The problem of identifying key variables necessary to understand the response  function (called sufficient dimensionality reduction in~\cite{li2018sufficient}) has been investigated in depth in the statistical literature.
In particular, estimates of the gradient of the target function can be used to identify relevant coordinates for the target function \cite{xia2002adaptive, RecursiveMultiIndex}. A series of works proposed methods that simultaneously learn the regression function and its gradient by non-parametric estimation~\cite{mukherjee2006learning, GradientEstimation}. The gradients can then be used to improve performance on downstream prediction tasks~\cite{GradientWeights}. Gradient estimation is particularly useful for coordinate selection in {\it multi-index models}, for which the regression function $f^*$ has the form $f^*(x) = g(Ux)$, where $U$ is a low rank matrix. Similar to neural feature learning, the multi-index estimator in~\cite{RecursiveMultiIndex} iteratively identifies the relevant subspace by learning the regression function and its gradient, but makes use of kernel smoothers.  {A line of recent work identifies the benefits of using neural networks for such multi-index (or single-index) problems by analyzing networks after 1 step of gradient descent~\cite{ba2022high, DLSRepresentationReLU} or showing that networks identify the principal subspace, $U$, through multiple steps of training~\cite{bietti2022learning,  NeuralNetsMultiIndexSGD}.}  Another line of work \cite{PHD} estimates the Principal Hessian Directions, the eigenvectors of the average Hessian matrix of the target function, to identify relevant coordinates.  Finally, a parallel line of research on ``active subspace'' methods in the context of \textit{dynamical systems}  has recently become a topic of active investigation~\cite{constantine2015active}.

\item  \textit{Metric and manifold learning}.  Updating feature matrices can also be viewed as learning a data-dependent  Mahalanobis distance, i.e. a distance $d_M(x, z) = \sqrt{ (x - z)^T M (x - z)}$ where $M$ is the feature matrix. This connects to a large body of literature on metric learning with numerous applications to various supervised and unsupervised learning problems~\cite{MetricLearningBook}.  Furthermore, we believe that feature learning methods such as neural networks or RFMs may benefit from incorporating ideas from the unsupervised and semi-supervised manifold learning and nonlinear dimensionality reduction literature~\cite{roweis2000nonlinear,belkin2003laplacian}. We also note that some of the early work on Radial Basis Networks explicitly addressed metric learning as a part of kernel function construction~\cite{poggio1990networks}.

\item \textit{FisherFaces and EigenFaces.} We further note the strong similarity between the eigenvectors of feature matrices (e.g., Figs.~\ref{fig: Overview},~\ref{appendix fig: NN and RFM features CelebA}) analyzed in this work and those given by EigenFace~\cite{EigenFace1, EigenFaces2}, and FisherFace~\cite{FisherFaces} algorithms.  While EigenFaces are obtained in a purely unsupervised fashion, the FisherFace algorithm uses labeled images of faces and Fisher's Linear Discriminant~\cite{FisherLinearDiscriminant} to learn a linear subspace for dimensionality reduction.  The first layer of neural networks and RFMs also learn linear subspaces based on labeled data but in a recursive way, using nonlinear classifiers. 

\item \textit{Debiasing.} Debiasing is a statistical procedure of recent interest in the statistics literature~\cite{zhang2014confidence}.  Given a high-dimensional problem with a hidden low-dimensional structure, debiasing involves first performing variable selection by using methods such as Lasso~\cite{tibshirani1996regression} or sparse PCA~\cite{jankova2021biased} and then fitting a low-dimensional model to the selected coordinates. We note that this procedure is similar to a single step of RFM. Moreover, both RFMs and neural networks can be viewed as a non-linear iterative version of the debiasing procedure with soft coordinate selection.

\item \textit{Expectation Maximization (EM).}  The RFM algorithm is reminiscent of the EM algorithm~\cite{moon1996expectation} with alternating estimation of the kernel predictor (M-step) and the feature matrix $M$ (E-step).  From this viewpoint, developing estimators for the feature matrix other than the sample covariance estimator considered in this work is an interesting future direction.  Moreover, depending on properties of the data and the target function, the feature matrix may be structured. Such structure could be leveraged to develop more sample efficient estimators for the M-step.

\item \textit{Boosting.} The mechanism of neural feature learning is reminiscent of boosting~\cite{freund1997decision} where a ``weak learner,'' only slightly correlated with the optimal predictor, is ``boosted'' by repeated application. Feature learning can similarly improve a suboptimal predictor as long as its average gradient outer product estimate is above the noise level. 
\end{itemize}

\paragraph{Looking forward.}  Overall, our work provides new insights into the operational principles of neural networks and how such principles can be leveraged to design new models with improved performance, computational simplicity, and transparency.  We envision that the mechanism of neural feature learning identified in this work will be key to improving neural networks and developing such new models.

\section*{Data Availability}
\label{sec: data availabillity}
All image datasets considered in this work, i.e., CelebA, SVHN, CIFAR10, MNIST and STL-10, are publicly available for download via PyTorch.  Tabular data from~\cite{FernandezDelgado} is available to download via \url{https://github.com/LeoYu/neural-tangent-kernel-UCI} provided by~\cite{NTKSmallDatasets}.  Tabular data from~\cite{TabularDataBenchmark} is available to download via \url{https://github.com/LeoGrin/tabular-benchmark}.  

\section*{Code Availability}
Code for neural network experiments is available at \url{https://github.com/aradha/deep_neural_feature_ansatz}.  Code for RFMs is available at \url{https://github.com/aradha/recursive_feature_machines/tree/pip_install}. 

\label{sec: code availability}

\section*{Acknowledgements} 
A.R. is supported by the Eric and Wendy Schmidt Center at the Broad Institute.  A.R. thanks Caroline Uhler for support on this work, and A.R., M.B. thank her for many insightful discussions over the years. 
We are grateful to Phil Long for useful insight into centering the \ajop. We thank Daniel Hsu, Joel Tropp and Jason Lee for valuable literature references.
We acknowledge support from the National Science Foundation (NSF) and the Simons Foundation for the Collaboration on the Theoretical Foundations of Deep Learning\footnote{\url{https://deepfoundations.ai/}} through awards DMS-2031883 and \#814639 as well as the  TILOS institute (NSF CCF-2112665). This work used the programs (1) XSEDE (Extreme science and engineering discovery environment)  which is supported by NSF grant numbers ACI-1548562, and (2) ACCESS (Advanced cyberinfrastructure coordination ecosystem: services \& support) which is supported by NSF grants numbers \#2138259, \#2138286, \#2138307, \#2137603, and \#2138296. Specifically, we used the resources from SDSC Expanse GPU compute nodes, and NCSA Delta system, via allocations TG-CIS220009.

\bibliographystyle{abbrv}
\bibliography{aux/references}

\appendix 
\newpage 

\section{Theoretical evidence for Deep Neural Feature Ansatz}
\label{appendix: Theoretical results}

We present the proof of Proposition~\ref{prop: 1 example} below.  
\begin{proof}
Gradient descent with learning rate $\eta$ proceeds as follows: 
\begin{align*}
    B^{(t+1)} = B^{(t)} + \eta \nabla g(B^{(t)}x) (y - g(B^{(t)}x)) x^T ~.
\end{align*}
If $B^{(0)} = \mathbf{0}$, then by induction $B^{(t)} = \alpha^{(t)} x^T$ for all time steps $t$ where $\alpha^{(t)} \in \mathbb{R}^{k}$.  Then, we have that
\begin{align*}
    \nabla f_t(z) \nabla f_t(z)^T &= {B^{(t)}}^T \nabla g(B^{(t)}z) g(B^{(t)}z)^T B^{(t)} \\
    &= x {\alpha^{(t)}}^T \nabla g(B^{(t)}z) g(B^{(t)}z)^T \alpha^{(t)} x^T \\ 
    &= (x x^T) ({\alpha^{(t)}}^T \nabla g(B^{(t)}z) g(B^{(t)}z)^T \alpha^{(t)}) \\
    &\propto x x^T ~.
\end{align*}
Similarly, we have that
\begin{align*}
    {B^{(t)}}^T B^{(t)} &= x {\alpha^{(t)}}^T \alpha^{(t)} x^T = (x x^T) ({\alpha^{(t)}}^T \alpha^{(t)} ) \propto x x^T ~. 
\end{align*}
\end{proof}

We now extend the proposition above to the setting where we have multiple training samples, and we train for one step of gradient descent.  

\begin{prop}
\label{prop: 1 step, multiple examples}
Let $f(z) = g(Bz)$ with $f: \mathbb{R}^{d} \to \mathbb{R}$ and $g: \mathbb{R}^{k} \to \mathbb{R}$. Assume $g(0) = 0$ and $\nabla g(0) \neq \mathbf{0}$.  Given $n$ training samples $\{(x_i, y_i)\}_{i=1}^n$, suppose that $f$ is trained to minimize $\frac{1}{2}\sum_{i=1}^{n}(y_i - f(x_i))^2$ using gradient descent.  If $B^{(0)} = \mathbf{0}$ and $f_1(z) := g(B^{(1)}z)$, then 
\begin{align*}
    \nabla f_1(z) \nabla f_1(z)^T \propto {B^{(1)}}^T B^{(1)}~.
\end{align*}
\end{prop}
\begin{proof}
Gradient descent proceeds as follows: 
\begin{align*}
    B^{(t+1)} = B^{(t)} + \eta \sum_{i=1}^{n} \nabla g(B^{(t)}x_i) (y_i - g(B^{(t)}x_i)) x_i^T ~.
\end{align*}
If $B^{(0)} = \mathbf{0}$, then $B^{(1)} = \eta \nabla g(0) \sum_{i=1}^{n} y_i x_i^T$.  Thus, we have that 
\begin{align*}
    \nabla f_1(z) \nabla f_1(z)^T &= {B^{(1)}}^T \nabla g(B^{(1)}z) g(B^{(1)}z)^T B^{(t)} \\
    &= \eta^2 \left(  \sum_{i=1}^{n} y_i x_i \nabla g(0)^T \right) \nabla g(B^{(1)}z) g(B^{(1)}z)^T  \left( \nabla g(0) \sum_{i=1}^{n} y_i x_i^T \right) \\ 
    &= \eta^2 \left(  \sum_{i=1}^{n} y_i x_i \right) \left( \sum_{i=1}^{n} y_i x_i^T \right)  \left( \nabla g(0)^T \nabla g(B^{(1)}z) g(B^{(1)}z)^T  \nabla g(0) \right) \\
    &\propto \left(  \sum_{i=1}^{n} y_i x_i \right) \left( \sum_{i=1}^{n} y_i x_i^T \right) ~.
\end{align*}
Similarly, we have: 
\begin{align*}
    {B^{(1)}}^T B^{(1)} &= \eta^2 \left(  \sum_{i=1}^{n} y_i x_i \nabla g(0)^T \right) \left( \nabla g(0) \sum_{i=1}^{n} y_i x_i^T \right) \\
    &= \eta^2 \left(  \sum_{i=1}^{n} y_i x_i  \right) \left( \sum_{i=1}^{n} y_i x_i^T \right) \| \nabla g(0) \|_2 \\
    &\propto \left(  \sum_{i=1}^{n} y_i x_i  \right) \left( \sum_{i=1}^{n} y_i x_i^T \right)~. 
\end{align*}
\end{proof}

Following a similar argument to that of Proposition~\ref{prop: 1 example}, we now prove the ansatz for 1-hidden layer linear neural networks.

\begin{prop}
\label{prop:linear}
    Let $f: \mathbb{R}^{d} \to \mathbb{R}$ denote a two layer neural network of the form 
    $$f(x) = A \frac{1}{\sqrt{k}} Bx ; $$
    where $A \in \mathbb{R}^{1 \times k}, B \in \mathbb{R}^{k \times d}$.  Suppose only $B$ is trainable.  Let $B^{(t)}$ and $f^{(t)}$ denote updated weights after $t$ steps of gradient descent on the dataset $(X, y) \in \mathbb{R}^{d \times n} \times \mathbb{R}^{1 \times n}$ with constant learning rate $\eta > 0$. If $\{A_i^{(0)}\}_{i=1}^{k}$ are i.i.d. random variables $\mathbb{E}[A_i^2] = 1$ and $B^{(0)} = \mathbf{0}$,     
    \begin{align*}
        \lim_{k \to \infty} {B^{(t)}}\tran B^{(t)} = \lim_{k \to \infty} \nabla {f^{(t)}} \nabla {f^{(t)}}\tran ~; 
    \end{align*}
    where $\nabla {f^{(t)}}$ is the gradient of ${f^{(t)}}$.\footnote{Note that since ${f^{(t)}}$ is linear, the gradient is constant and independent of the point at which it is taken.}  
\end{prop}

\begin{proof}[Proof of Proposition~\ref{prop:linear}]
    The gradient descent updates proceed as follows: 
    \begin{align*}
        B^{(t+1)} &= B^{(t)} +  \frac{\eta}{\sqrt{k}} {A^{(0)}}\tran \left(y - A^{(0)} \frac{1}{\sqrt{k}} B^{(t)} X \right) X\tran    ~.
    \end{align*}
    We provide a proof by induction.  We begin with the base case with $t = 1$.  The base case follows from the fact that $\lim\limits_{k \to \infty} \frac{1}{k} A^{(0)} {A^{(0)}}\tran = 1$ and $B^{(1)} = \frac{\eta}{\sqrt{k}} {A^{(0)}}\tran y X\tran$ and thus, 
    \begin{align*}
        \lim_{k \to \infty} {B^{(1)}}\tran B^{(1)} &= \lim_{k \to \infty} \frac{\eta^2}{k} X y\tran A^{(0)} {A^{(0)}}\tran y X\tran = \eta^2 X y\tran y X\tran~; \\
        \lim_{k \to \infty} \nabla f_1 \nabla f_1\tran &= \lim_{k \to \infty} {B^{(1)}}\tran {A^{(0)}}\tran \frac{1}{k} A^{(0)} B^{(1)} \\
        &= \lim_{k \to \infty} \eta^2  X y\tran \left(A^{(0)} \frac{1}{k} {A^{(0)}}^{T}\right) \left(A^{(0)} \frac{1}{k} {A^{(0)}}\tran \right)  y X\tran = \eta^2  X y\tran y X\tran~. 
    \end{align*}
    Thus, we now assume the inductive hypothesis that
    \begin{align*}
        \lim_{k \to \infty} {B^{(t)}}\tran B^{(t)} = \lim_{k \to \infty} \nabla {f^{(t)}} \nabla {f^{(t)}}\tran~
    \end{align*}    
    and analyze the case for timestep $t+1$.  We first have: 
    \begin{align*}
        {B^{(t+1)}}\tran B^{(t+1)} &
         =  \left[ {B^{(t)}} + \frac{\eta}{\sqrt{k}} {A^{(0)}}\tran \left(y - A^{(0)} \frac{1}{\sqrt{k}} B^{(t)} X \right) X\tran  \right]\tran \left[ {B^{(t)}} + \frac{\eta}{\sqrt{k}} {A^{(0)}}\tran \left(y - A^{(0)} \frac{1}{\sqrt{k}} B^{(t)} X \right) X\tran\right]\\
         &=  {B^{(t)}}\tran {B^{(t)}} + {B^{(t)}}\tran \frac{\eta}{\sqrt{k}} {A^{(0)}}\tran y X\tran - {B^{(t)}}\tran \frac{\eta}{k} {A^{(0)}}\tran {A^{(0)}} {B^{(t)}}X X\tran \\
         &\qquad +  \frac{\eta}{\sqrt{k}} X y\tran {A^{(0)}} B^{(t)} + \frac{\eta^2}{k} X y\tran {A^{(0)}} {A^{(0)}}\tran y X\tran - \frac{\eta^2}{k} X y\tran {A^{(0)}}  {A^{(0)}}\tran A^{(0)} \frac{1}{\sqrt{k}} B^{(t)} X X\tran \\
         &\qquad - \frac{\eta}{k} X X\tran {B^{(t)}}\tran  {A^{(0)}}\tran A^{(0)}  B^{(t)} - \frac{\eta^2}{k}  X X\tran {B^{(t)}}\tran \frac{1}{\sqrt{k}} {A^{(0)}}\tran A^{(0)} {A^{(0)}}\tran y X\tran  \\
         &\qquad + \frac{\eta^2}{k^2}  X X\tran {B^{(t)}}\tran  {A^{(0)}}\tran A^{(0)}  {A^{(0)}}\tran A^{(0)}  B^{(t)} X X\tran~. 
    \end{align*}
    To simplify notation, we let 
    \begin{align*}
        Z = \lim_{k \to \infty} A^{(0)} \frac{1}{\sqrt{k}} B^{(t)} \qquad ; \qquad M = \lim_{k \to \infty} {B^{(t)}}\tran B^{(t)}, 
    \end{align*}
    noting that for $x \in \mathbb{R}^{d}$, $Zx$ converges in distribution to a standard normal random variable by the central limit theorem.  Taking the limit as $k \to \infty$, applying the inductive hypothesis and the fact that $\lim\limits_{k \to \infty} \frac{1}{k} A^{(0)} {A^{(0)}}\tran = 1$, we reduce the above to 
    \begin{align*}
        \lim_{k \to \infty} {B^{(t+1)}}\tran B^{(t+1)} &= M + \eta Z\tran y X\tran - \eta M X X\tran \\
         &\qquad +  \eta X y\tran Z  + \eta^2 X y\tran  y X\tran - \eta^2  X y\tran Z X X\tran \\
         &\qquad - \eta X X\tran M - \eta^2  X X\tran Z\tran y X\tran + \eta^2  X X\tran M X X\tran~. 
    \end{align*}
    We will now show that $\lim\limits_{k \to \infty} \nabla f^{(t+1)} \nabla {f^{(t+1)}}\tran$ is of the same form.  Namely, we have 
    \begin{align*}
        \nabla f^{(t+1)} \nabla {f^{(t+1)}}\tran &= {B^{(t+1)}}\tran {A^{(0)}}\tran \frac{1}{k} A^{(0)} B^{(t+1)} \\
        &= {B^{(t)}}\tran {A^{(0)}}\tran \frac{1}{k} A^{(0)} {B^{(t)}} + {B^{(t)}}\tran {A^{(0)}}\tran \frac{1}{k} A^{(0)} \frac{\eta}{\sqrt{k}} {A^{(0)}}\tran y X\tran \\
        &\qquad - {B^{(t)}}\tran {A^{(0)}}\tran \frac{1}{k} A^{(0)} \frac{\eta}{k} {A^{(0)}}\tran {A^{(0)}} {B^{(t)}}X X\tran \\
         &\qquad +  \frac{\eta}{\sqrt{k}} X y\tran {A^{(0)}}  {A^{(0)}}\tran \frac{1}{k} A^{(0)} B^{(t)} + \frac{\eta^2}{k} X y\tran {A^{(0)}}  {A^{(0)}}\tran \frac{1}{k} A^{(0)} {A^{(0)}}\tran y X\tran \\
         &\qquad - \frac{\eta^2}{k} X y\tran {A^{(0)}}  {A^{(0)}}\tran \frac{1}{k} A^{(0)} {A^{(0)}}\tran A^{(0)} \frac{1}{\sqrt{k}} B^{(t)} X X\tran \\
         &\qquad - \frac{\eta}{\sqrt{k}} X X\tran {B^{(t)}}\tran  {A^{(0)}}\tran A^{(0)}  {A^{(0)}}\tran \frac{1}{k} A^{(0)} B^{(t)} \\
         &\qquad - \frac{\eta^2}{k}  X X\tran {B^{(t)}}\tran \frac{1}{\sqrt{k}} {A^{(0)}}\tran A^{(0)}  {A^{(0)}}\tran \frac{1}{k} A^{(0)} {A^{(0)}}\tran y X\tran  \\
         &\qquad + \frac{\eta^2}{k^2}  X X\tran {B^{(t)}}\tran  {A^{(0)}}\tran A^{(0)}   {A^{(0)}}\tran \frac{1}{k} A^{(0)} {A^{(0)}}\tran A^{(0)}  B^{(t)} X X\tran~. 
    \end{align*}
Now taking the limit as $k \to \infty$, we reduce the above to
\begin{align*}
    \lim_{k \to \infty} \nabla f^{(t+1)} \nabla {f^{(t+1)}}\tran &= M + \eta Z\tran y X\tran - \eta M XX\tran \\
    &\qquad + \eta X y\tran Z + \eta^2 X y\tran y X\tran - \eta^2 X y\tran Z X X\tran \\
    &\qquad - \eta XX\tran M -\eta^2 XX\tran Z\tran y X\tran +  \eta^2 XX\tran M XX\tran~.
\end{align*}
Hence, we conclude
\begin{align*}
    \lim_{k \to \infty} {B^{(t+1)}}\tran B^{(t+1)}  = \lim_{k \to \infty} \nabla f^{(t+1)} \nabla {f^{(t+1)}}\tran, 
\end{align*}
which completes the proof by induction. 
\end{proof}

In the following proposition, we extend the previous analysis to the case of two layer linear neural networks where both layers are trained for two steps of gradient descent.     

\begin{prop}
\label{prop:layerbothlayers}
    Let $f: \mathbb{R}^{d} \to \mathbb{R}$ denote a two layer neural network of the form 
    $$f(x) = A \frac{1}{\sqrt{k}} Bx ; $$
    where $A \in \mathbb{R}^{1 \times k}, B \in \mathbb{R}^{k \times d}$.  Let $A^{(t)}, B^{(t)}$ and $f^{(t)}$ denote updated weights after $t$ steps of gradient descent on the dataset $(X, y) \in \mathbb{R}^{d \times n} \times \mathbb{R}^{1 \times n}$ with constant learning rate $\eta > 0$. If $\{A_i^{(0)}\}_{i=1}^{k}$ are i.i.d. random variables $\mathbb{E}[{A_i^{(0)}}^2] = 1$ and $B^{(0)} = \mathbf{0}$,     
    \begin{align*}
        \lim_{k \to \infty} {B^{(2)}}\tran B^{(2)} = \lim_{k \to \infty} \nabla {f^{(2)}} \nabla {f^{(2)}}\tran ~; 
    \end{align*}
    where $\nabla {f^{(2)}}$ is the gradient of ${f^{(2)}}$.
\end{prop}

\begin{proof}
    We prove the statement directly.  The gradient descent updates proceed as follows: 
    \begin{align*}
        A^{(t+1)} &= A^{(t)} + \frac{\eta}{\sqrt{k}} \left(y - A^{(t)} \frac{1}{\sqrt{k}} B^{(t)}X\right)X\tran {B^{(t)}}\tran~,  \\
        B^{(t+1)} &= B^{(t)} + \frac{\eta}{\sqrt{k}} {A^{(t)}}\tran \left(y - A^{(t)} \frac{1}{\sqrt{k}} B^{(t)}X\right)X\tran~.
    \end{align*}
    Thus, after 1 step of gradient descent, we have
    \begin{align*}
        A^{(1)} = A^{(0)} \qquad ; \qquad B^{(1)} = \frac{\eta}{\sqrt{k}} {A^{(0)}}\tran y X\tran . 
    \end{align*}
    From the proof of Proposition~\ref{prop:linear}, we have that
    \begin{align*}
        \lim_{k \to \infty}{B^{(1)}}\tran B^{(1)}  = \lim_{k \to \infty} {B^{(1)}}\tran \frac{1}{k} {A^{(0)}}\tran A^{(0)} B^{(1)}~, 
    \end{align*}
    and so, we define the matrix $M$ to be: 
    \begin{align*}
        M := \lim_{k \to \infty}{B^{(1)}}\tran B^{(1)}.
    \end{align*}
    Next, after 2 steps of gradient descent, we have: 
    \begin{align*}
        A^{(2)} &= A^{(1)} + \frac{\eta}{\sqrt{k}} \left(y - A^{(1)} \frac{1}{\sqrt{k}} B^{(1)}X \right)  X\tran {B^{(1)}}\tran \\
        &= A^{(0)} + \frac{\eta}{\sqrt{k}} y X\tran {B^{(1)}}\tran - \frac{\eta}{k} A^{(0)}  B^{(1)}X   X\tran {B^{(1)}}\tran \\
        & = A^{(0)} +  \frac{\eta^2}{k} y X\tran  X y\tran {A^{(0)}} -  \frac{\eta^3}{k^2} A^{(0)} {A^{(0)}}\tran y X\tran X X\tran  X y\tran {A^{(0)}}; \\
    \end{align*}
    and
    \begin{align*}
        B^{(2)} &= B^{(1)}  + \frac{\eta}{\sqrt{k}} {A^{(1)}}\tran \left(y - A^{(1)} \frac{1}{\sqrt{k}} B^{(1)}X\right)X\tran \\
        &= \frac{2\eta}{\sqrt{k}} {A^{(0)}}\tran y X\tran - \frac{\eta}{k} {A^{(0)}}\tran A^{(0)} B^{(1)} X X\tran \\
        &= \frac{2\eta}{\sqrt{k}} B^{(1)} - \frac{\eta}{k} {A^{(0)}}\tran A^{(0)} B^{(1)} X X\tran~.
    \end{align*}
    Thus, we simplify ${B^{(2)}}\tran B^{(2)}$ as follows:
    \begin{align*}
        {B^{(2)}}\tran B^{(2)} &= \frac{4\eta^2}{k} {B^{(1)}}\tran {B^{(1)}} - \frac{2\eta^2}{k\sqrt{k}} {B^{(1)}}\tran {A^{(0)}}\tran A^{(0)} B^{(1)} X X\tran \\
        &\qquad - \frac{2\eta^2}{k\sqrt{k}} X X\tran {B^{(1)}}\tran {A^{(0)}}\tran A^{(0)} {B^{(1)}} + \frac{\eta^2}{k^2}  X X\tran {B^{(1)}}\tran {A^{(0)}}\tran A^{(0)} {A^{(0)}}\tran A^{(0)} B^{(1)} X X\tran~.
    \end{align*}
    Taking the limit as $k \to \infty$, we simplify the above expression to
    \begin{align*}
        \lim_{k \to \infty} {B^{(2)}}\tran B^{(2)} = 4 \eta^2 M + \eta^2 XX\tran M XX\tran~.
    \end{align*}
    
    A key observation is that as $k \to \infty$, the $O\left( \frac{1}{k} \right)$ and $O\left( \frac{1}{k^2} \right)$ terms in $A^{(2)}$ will vanish in the evaluation of $\lim\limits_{k \to \infty} \nabla {f^{(2)}} \nabla {f^{(2)}}\tran$ since the gradient also contains an extra $\frac{1}{\sqrt{k}}$ term from $f$.  Hence only the $O(1)$ terms given by $A^{(0)}$ will remain in the evaluation of $\lim\limits_{k \to \infty} \nabla {f^{(2)}} \nabla {f^{(2)}}\tran$. Using this observation, we have:
    \begin{align*}
        \lim\limits_{k \to \infty} \nabla f^{(2)} \nabla {f^{(2)}}\tran &= \lim_{k \to \infty} {B^{(2)}}\tran \frac{1}{k} {A^{(2)}}\tran A^{(2)} B^{(2)} = \lim_{k \to \infty} {B^{(2)}}\tran \frac{1}{k} {A^{(0)}}\tran A^{(0)} B^{(2)},  
    \end{align*}
    which by the expansion of ${B^{(2)}}\tran B^{(2)}$ and the proof of Proposition~\ref{prop:linear}, is equivalent to $4 \eta^2 M + \eta^2 XX\tran M XX\tran$.  
\end{proof}

The above results demonstrate that the ansatz holds for one-hidden-layer neural networks trained in isolation.  We now prove the ansatz in the more general setting of deep, nonlinear fully connected networks by ensembling, or averaging over infinitely many networks.  We present the proof of Theorem~\ref{thm:nonlinear_resampled} below.

\begin{proof}
For a matrix $A_\alpha \in \Real^{c \times d}$, we denote its $p$-th row by $A_{\alpha,p}$. For a vector $v \in \Real^d$, we denote its $i$-th element by $v(i)$. To simplify notation, we drop the subscript $t$ if it is irrelevant (e.g., fixed) in an expression.  We consider the right hand side of the desired equation. The gradient with respect to the input is given by
\begin{align*}
    \nabla_x f(x) =  \frac{c_\phi^L}{\prod_{\ell=1}^{L} k_{\ell}} &\sum_{k'_{L}} \sum_{k'_{L-1}} \ldots  \sum_{k'_{1}} a_{k'_{L}} \phi'({W_{L,k'_{L}}}\tran \phi_{L-1}(z))\\
    &\cdot W_{L,k'_{L}}(k'_{L-1}) \phi'({W_{L-1,k'_{L-1}}}\tran \phi_{L-2}(x)) \cdots W_{3,k'_{3}}(k'_{2}) \cdot \phi'({W_{2,k'_{2}}}\tran \phi_{1}(x)) \\
    &\cdot W_{2,k'_{2}}(k'_1) \phi'({W_{1,k'_1}}\tran x) \cdot W_{1,k'_1} ~.
\end{align*}
Then,
\begin{align*}
    \nabla_x f(x) {\nabla_x f(x)}\tran &=  \frac{c_\phi^L}{\prod_{\ell=1}^{L} k_{\ell}} \\
    &\cdot \sum_{k'_{L},k''_{L}} \sum_{k'_{L-1},k''_{L-1}} \ldots \sum_{k'_{1},k''_{1}} \phi'\round{{W_{L,k'_{L}}}\tran \phi_{L-1}(x)} \cdots \phi'\round{{W_{2,k'_{2}}}\tran \phi(W_1 x)} \phi'(({W_{1,k'_1})}\tran x)\\
    &\cdot \phi'\round{{W_{L,k''_{L}}}\tran \phi_{L-1}(x)} \cdots \phi'\round{{W_{2,k''_{2}}}\tran \phi(W_1 x)} \phi'(({W_{1,k''_1}})\tran x)\\
    & \cdot a_{k'_{L}} W_{L,k'_{L}}(k'_{L-1}) \cdots W_{3,k'_{3}}(k'_{2}) W_{2,k'_2}(k'_1) \\
    & \cdot a_{k''_{L}} W_{L,k''_{L}}(k''_{L-1}) \cdots W_{3,k''_{3}}(k''_{2}) W_{2,k''_2}(k''_1)\\
    & \cdot W_{1,k'_1} {W_{1,k''_1}}\tran~.
 \end{align*}
 By the gradient independence ansatz, we can generate new samples $\wt{W}_L, \ldots, \wt{W}_2$,
 \begin{align*}
     \lim_{k_2,\ldots,k_L \to \infty} \mathbb{E}_a \left[ \nabla_x f(x) {\nabla_x f(x)}\tran \right] &=  \lim_{k_2,\ldots,k_L \to \infty} \mathbb{E}_a \frac{c_\phi^L}{\prod_{\ell=1}^{L} k_{\ell}} \sum_{k'_{L},k''_{L}} \sum_{k'_{L-1},k''_{L-1}} \ldots \sum_{k'_{1},k''_{1}} \\
    &\cdot \phi'\round{{W_{L,k'_{L}}}\tran \phi_{L-1}(x)} \cdots \phi'\round{{W_{2,k'_{2}}}\tran \phi(W_1 x)} \phi'(({W_{1,k'_1})}\tran x)\\
    &\cdot \phi'\round{{W_{L,k''_{L}}}\tran \phi_{L-1}(x)} \cdots \phi'\round{{W_{2,k''_{2}}}\tran \phi(W_1 x)} \phi'(({W_{1,k''_1}})\tran x)\\
    & \cdot a_{k'_{L}} \wt{W}_{L,k'_{L}}(k'_{L-1}) \cdots \wt{W}_{3,k'_{3}}(k'_{2}) \wt{W}_{2,k'_2}(k'_1) \\
    & \cdot a_{k''_{L}} \wt{W}_{L,k''_{L}}(k''_{L-1}) \cdots \wt{W}_{3,k''_{3}}(k''_{2}) \wt{W}_{2,k''_2}(k''_1)\\
    & \cdot W_{1,k'_1} {W_{1,k''_1}}\tran~.
\end{align*}
Pulling factors outside of the limit,
\begin{align*}
    \lim_{k_2,\ldots,k_L \to \infty} \mathbb{E}_a \left[ \nabla_x f(x) {\nabla_x f(x)}\tran \right]  &=  \lim_{k_{2},\ldots,k_{L-1} \to \infty} \frac{c_\phi^L}{\prod_{\ell=1}^{L-1} k_{\ell}} \\
    & \lim_{k_L \to \infty} \frac{1}{k_L} \sum_{k'_{L}} \round{\phi'\round{{W_{L,k'_{L}}}\tran \phi_{L-1}(x)}}^2 \\
    &\sum_{k'_{L-1},k''_{L-1}}  \wt{W}_{L,k'_{L}}(k'_{L-1})  \wt{W}_{L,k'_{L}}(k'_{L-1}) \\
    &\cdot \sum_{k'_{L-2},k''_{L-2}} \ldots \sum_{k'_{1},k''_{1}} \phi'\round{{W_{L-1,k'_{L-1}}}\tran \phi_{L-2}(x)} \cdots \phi'(({W_{1,k'_1})}\tran x)\\
    &\cdot \phi'\round{{W_{L-1,k''_{L-1}}}\tran \phi_{L-2}(x)} \cdots \phi'(({W_{1,k''_1}})\tran x)\\
    & \cdot \wt{W}_{L-1,k'_{L-1}}(k'_{L-2}) \cdots \wt{W}_{3,k'_{3}}(k'_{2}) \wt{W}_{2,k'_2}(k'_1) \\
    & \cdot \wt{W}_{L-1,k''_{L-1}}(k''_{L-2}) \cdots \wt{W}_{3,k''_{3}}(k''_{2}) \wt{W}_{2,k''_2}(k''_1)\\
    & \cdot W_{1,k'_1} {W_{1,k''_1}}\tran~.
 \end{align*}
 Note that by re-sampling, $\round{\phi'\round{{W_{L,k'_{L}}}\tran \phi_{L-1}(x)}}^2$ is independent of the remaining terms, and so we can apply the law of large numbers as $k_L \to \infty$ and split the expectation as follows.
\begin{align*}
 \lim_{k_2,\ldots,k_L \to \infty} \mathbb{E}_a \left[\nabla_x f(x) {\nabla_x f(x)}\tran \right]
    &=  \lim_{k_{2},\ldots,k_{L-1} \to \infty} \frac{c_\phi^L}{\prod_{\ell=1}^{L-1} k_{\ell}} \\
    & \mathbb{E}_{k'_{L}} \sbrac{\round{\phi'\round{{W_{L,k'_{L}}}\tran \phi_{L-1}(x)}}^2} \\
    &\mathbb{E}_{k'_{L}} \Bigg[ \sum_{k'_{L-1},k''_{L-1}}  \wt{W}_{L,k'_{L}}(k'_{L-1})  \wt{W}_{L,k'_{L}}(k'_{L-1}) \\
    &\cdot \sum_{k'_{L-2},k''_{L-2}} \ldots \sum_{k'_{1},k''_{1}} \phi'\round{{W_{L-1,k'_{L-1}}}\tran \phi_{L-2}(x)} \cdots \phi'(({W_{1,k'_1})}\tran x)\\
    &\cdot \phi'\round{{W_{L-1,k''_{L-1}}}\tran \phi_{L-2}(x)} \cdots  \phi'(({W_{1,k''_1}})\tran x)\\
    & \cdot \wt{W}_{L-1,k'_{L-1}}(k'_{L-2}) \cdots \wt{W}_{3,k'_{3}}(k'_{2}) \wt{W}_{2,k'_2}(k'_1) \\
    & \cdot \wt{W}_{L-1,k''_{L-1}}(k''_{L-2}) \cdots \wt{W}_{3,k''_{3}}(k''_{2}) \wt{W}_{2,k''_2}(k''_1)\\
    & \cdot W_{1,k'_1} {W_{1,k''_1}}\tran \Bigg]~.
\end{align*}
Evaluating the expectations above, we conclude:
\begin{align*}
     \lim_{k_2,\ldots,k_L \to \infty} \mathbb{E}_a \left[ \nabla_x f(x) {\nabla_x f(x)}\tran \right] 
    &=  \lim_{k_{2},\ldots,k_{L-1} \to \infty} \frac{c_\phi^{L-1}}{\prod_{\ell=1}^{L-1} k_{\ell}} \\
    & \cdot \sum_{k'_{L-1}} \sum_{k'_{L-2},k''_{L-2}} \ldots \sum_{k'_{1},k''_{1}} \phi'\round{{W_{L-1,k'_{L-1}}}\tran \phi_{L-2}(x)} \cdots  \phi'(({W_{1,k'_1})}\tran x)\\
    &\cdot \phi'\round{{W_{L-1,k'_{L-1}}}\tran \phi_{L-2}(x)} \cdots  \phi'(({W_{1,k''_1}})\tran x)\\
    & \cdot \wt{W}_{L-1,k'_{L-1}}(k'_{L-2}) \cdots \wt{W}_{3,k'_{3}}(k'_{2}) \wt{W}_{2,k'_2}(k'_1) \\
    & \cdot \wt{W}_{L-1,k''_{L-1}}(k''_{L-2}) \cdots \wt{W}_{3,k''_{3}}(k''_{2}) \wt{W}_{2,k''_2}(k''_1)\\
    & \cdot W_{1,k'_1} {W_{1,k''_1}}\tran ~.
\end{align*}
 Recursively applying this procedure yields
\begin{align*}
     \lim_{k_2,\ldots,k_L \to \infty} \mathbb{E}_a \left[ \nabla_x f(x) {\nabla_x f(x)}\tran \right]
    &=  \frac{c_\phi}{k_{1}} \sum_{k'_{1}} \phi'(({W_{1,k'_1})}\tran x)^2 W_{1,k'_1} {W_{1,k'_1}}\tran~.
\end{align*}
Taking the expectation with respect to $x$,
\begin{align*}
    \mathbb{E}_{x} \sbrac{\lim_{k_2,\ldots,k_L \to \infty} \mathbb{E}_{a} \left[ \nabla_x f(x) {\nabla_x f(x)}\tran \right] } 
    &= \frac{1}{k_1} \sum_{k'_1} W_{1,k'_1} {W_{1,k'_1}}\tran~.
\end{align*}
\end{proof}

\section{Methods} 
\label{appendix: Methods}

Below, we provide a description of all datasets, models, and training methodology considered in this work.    

\paragraph{Validating the Deep Neural Feature Ansatz} 

All neural networks in Fig.~2A have $5$ hidden layers with $1024$ hidden units per layer and ReLU activation.  We use minibatch gradient descent with a learning rate of $0.1$ and batch size $128$ for $500$ epochs and initialize the first layer weights according to a Gaussian distribution with mean $0$ and standard deviation $10^{-4}$.  All networks used in Fig.~2B and Supplementary Fig.~\ref{appendix fig: Ansatz CelebA SVHN} have $5$ hidden layers with $64$ hidden units per layer and ReLU activation.  We use minibatch gradient descent with a learning rate of $0.2$ and batch size of $128$ for $500$ epochs and initialize the first layer weights according to a Gaussian distribution with mean $0$ and standard deviation of $10^{-6}$. 

\paragraph{Spurious features.} For experiment in Fig.~\ref{fig: Spurious features}A, we constructed a training set of size $50000$ concatenated CIFAR-10 and MNIST digits, and a corresponding test set of $10000$ test images. The training and test data were generated from data loaders provided by PyTorch.  We used 20$\%$ of the training samples were used for validation.  For Fig.~\ref{fig: Spurious features}A and B, we trained a five-hidden-layer fully connected ReLU network with $64$ hidden units per layer using SGD with a learning rate of $0.2$ and a mini-batch size of $128$.  We initialized first layer weights from a Gaussian with mean zero and standard deviation $10^{-4}$.  For Fig.~\ref{fig: Spurious features}C, we trained a two-hidden-layer fully connected ReLU network with $5$ hidden unhits per layer using SGD for $500$ epochs with a learning rate of $0.1$ and a mini-batch size of $128$.  For all experiments, we train using the mean squared error (MSE) with one-hot labels for each of the classes.

\paragraph{Grokking.} The total number of training and validation samples used is $553$ with $500$ examples of airplanes and $53$ examples of trucks.  We use $800$ examples per class from the PyTorch test set as test data.  We set a small stars of pixels ($8$ pixels tall, $7$ pixels wide) in the upper left corner to yellow (all $1$s in the green and red channel) if the image is a truck and all $0$s if the image is a plane.  We use $80\%$ of the $553$ samples for training and $20\%$ for validation.   We train a two hidden layer fully connected ReLU network using Adam with a learning rate of $10^{-4}$ and batch size equal to dataset size.  We initialize the weights of the first layer of the ReLU network according to a normal distribution with standard deviation $5 \times 10^{-3}$.  We train RFMs updating only the diagonals of the feature matrix for three iterations with ridge regularization of $2.5 \times 10^{-3}$ and using the Laplace kernel as the base kernel function with a bandwidth of $10$.  We used ridge regularization to slow down training of RFMs to visualize how the feature matrix changes through iteration.  We note that without regularization, the RFM gets $100\%$ test accuracy within 1 iteration.

\paragraph{Lottery Tickets.}  For all binary classification tasks on CelebA, we normalize all images to be on the unit sphere.  We train 2-hidden layer ReLU networks with 1024 hidden units per layer using stochastic gradient descent (SGD) for $500$ epochs with a learning rate of $0.1$ and a mini-batch size of $128$.  We train using the mean squared error (MSE) with one-hot labels for each of the classes.  Accuracy is reported as the argmax across classes.  We split available training data into $80\%$ training and $20\%$ validation for hyper-parameter selection.  We report accuracy on a held out test set provided by PyTorch~\cite{PyTorch}.  In addition, since there can be large class imbalances in this data, we ensure that the training set and test set are balanced by limiting the number of majority class samples to the same number of minority class samples.  Given that these are higher resolution images, we limit the total number of training and validation examples per experiment to $50000$ ($25000$ per class).  When re-training networks after masking to the top $2\%$ of pixels with highest intensity in the diagonal of the first layer NFM, we re-initialize networks of the same architecture using the default PyTorch initialization scheme.   

\paragraph{RFMs trained on CelebA.} For the CelebA tasks in Appendix Fig.~\ref{appendix fig: NN and RFM features CelebA}, we train RFMs for 1 iteration, use a ridge regularization term of $10^{-3}$, and average the gradient outer product of at most $20000$ examples.  All RFMs use Laplace kernels as the base kernel and use a bandwidth parameter of $L = 10$.  We solve kernel ridge regression exactly via the solve function in numpy~\cite{numpy2}.  We use the same training, validation, and test splits considered in the lottery ticket experiments.

\paragraph{RFMs, neural networks, NTK, and Laplace kernels on SVHN.} In Appendix Fig.~\ref{appendix fig: NN and RFM features SVHN Low Rank}A, we train 2-hidden layer ReLU networks with 1024 hidden units per layer using stochastic gradient descent (SGD) for $500$ epochs with a learning rate of $0.05$ and a mini-batch size of $100$.  We train using the mean squared error (MSE) with one-hot labels for each of the classes.  Accuracy is reported as the argmax across classes.  We train RFMs for 5 iterations and average the gradient outer product of at most $20000$ examples.  We also center gradients during computation of RFMs by subtracting the mean of the gradients before computing the average gradient outer product.  RFMs and Laplace kernels used all have a bandwidth parameter of $10$.  We compare with the NTK of a 2-hidden layer ReLU network.  For all kernels, we solve kernel ridge regression with ridge term of $10^{-3}$ via the solve function in numpy~\cite{numpy2}.  The test accuracy for RFMs in Fig.~\ref{appendix fig: NN and RFM features SVHN Low Rank}A is given by training a 1-hidden layer NTK with ridge regularization of $10^{-2}$ on the feature matrix selected from the last iteration of training, which resulted in the best validation accuracy. 

\paragraph{RFMs, neural networks, NTK, and Laplace kernels on low rank polynomial regression.} In Appendix Fig.~\ref{appendix fig: NN and RFM features SVHN Low Rank}B and C, we consider the low rank polynomials from~\cite{PreetumLimitations} and~\cite{DLSRepresentationReLU}.  We use $1000$ examples for training and $10000$ samples for testing. Following the setup of~\cite{PreetumLimitations}, we sample training inputs from a Rademacher distribution in $30$ dimensions and add random noise (see Appendix Fig.~\ref{appendix fig: NN and RFM features SVHN Low Rank}B).  The labels are generated by the product of the first two coordinates of the inputs without noise.  We train a 1 hidden layer neural network for $1000$ epochs using full batch gradient descent with a learning rate of .1 and initialize the first layer with standard deviation $10^{-3}$ so as to mitigate the effect of the initialization. We train RFMs with no ridge term and set the base kernel function as the Laplace kernel with bandwidth 10.  We note the neural network was able to interpolate the training data and achieved a training $R^2$ of $1$.

For the second low rank experiment in Appendix Fig.~\ref{appendix fig: NN and RFM features SVHN Low Rank}C, we sample inputs, $x$,  according to a $10$ dimensional isotropic Gaussian distribution and sample a fixed vector, $u$,  on the unit sphere in $10$ dimensions.  The targets are given by $g(u^T x)$ where $g(z) = \text{He}_2(z) + \text{He}_4(z)$ where $\text{He}_2, \text{He}_4$ are the second and fourth probabilist's Hermite polynomials.  We train a 1 hidden layer neural network using full batch Adam~\cite{Adam} with a learning rate of $10^{-2}$ and use the default PyTorch initialization.  We train RFMs with no ridge term and set the base kernel function as the Laplace kernel with bandwidth 10.  We note the neural network was able to nearly interpolate the training data within $1000$ epochs and achieved a training $R^2$ of $0.971$.  

\paragraph{121 datasets from~\cite{FernandezDelgado}.} We first describe the experiments for $120$ of the $121$ datasets with fewer than $130000$ examples since we used EigenPro \cite{EigenProGPU} to train kernels on the largest dataset.  For all kernel methods (RFMs, Laplace kernel and NTK), we grid search over ridge regularization from the set $\{10, 1, .1, .01, .001\}$.  We grid search over 5 iterations for RFMs and used a bandwidth of $10$ for all Laplace kernels.  For NTK ridge regression experiments, we grid search over NTKs corresponding to ReLU networks with between $1$ and $5$ hidden layers.  For the dataset with $130000$ samples, we use EigenPro to train all kernel methods and RFMs.  We run EigenPro for at most 50 iterations and select the  iteration with best validation accuracy for reporting test accuracy.  For small datasets (i.e., those with fewer than $5000$ samples), we grid search over updating just the diagonals of $M$ and updating the entire matrix $M$.  Lastly, for all kernel methods and RFMs, we grid search over normalizing the data to the unit sphere.  We note that there is one dataset (balance-scale), which had a data point with norm $0$, and so we did not grid search over normalization for this dataset. 

\paragraph{Tabular data benchmark from~\cite{TabularDataBenchmark}.} We used the repository from~\cite{TabularDataBenchmark} at \url{https://github.com/LeoGrin/tabular-benchmark}, modifying the code as needed to incorporate our method. On all datasets, we grid search over $5$ iterations of RFM with the Laplace kernel, solving  kernel regression in closed form at all steps. This benchmark consists of $20$ medium regression datasets (without categorical variables), $3$ large regression datasets (without categorical variables), $15$ medium classification datasets (without categorical variables), $4$ large classification datasets (without categorical variables), $13$ medium classification datasets (with categorical variables), $5$ large regression datasets (with categorical variables), $7$ medium classification datasets (with categorical variables), and $2$ large classification datasets (with categorical variables). Following the terminology from~\cite{TabularDataBenchmark}, ``medium'' refers to datasets with at most $10000$ training examples and ``large'' refers to those with more than $10000$ training examples.  In general, we grid-searched over ridge regularization parameters in $\{10^{-4},10^{-3},10^{-2},10^{-1},1\}$ with fixed bandwidth $L = 10$. For regression, we centered the labels and scaled their variance to $1$. On large regression datasets, we also optimized for bandwidths over $\{1,5,10,15,20\}$. We searched over two target transformations - the log transform ($\wh{y} = |y| \log(1 + |y|)$) and {\tt sklearn.preprocessing.QuantileTransformer}. In both cases, we inverted the transform before testing. We also searched over data transformations - {\tt sklearn.preprocessing.StandardScaler} and {\tt sklearn.preprocessing.QuantileTransformer}. We also optimized for the use of centering/not centering the gradients in our computation, and extracting just the diagonal of the feature matrix.  For non-kernel methods, we compare to the metrics reported in \cite{TabularDataBenchmark}. For classification, we report the average accuracy across the random iterations in each sweep (including random train/val/test splits). For regression, the average $R^2$ is reported. The reported test score is the average performance of the model with the highest average validation performance.

\vspace{3mm} 

\noindent We next provide a description of all metrics considered in the tabular benchmarks.

\paragraph{Friedman Rank.}  To compute Friedman rank, we rank classifiers in order of performance (e.g. the top performer gets rank $1$) for each dataset and then average the ranks.  In the original results of~\cite{FernandezDelgado}, certain classifiers were missing performance values.  To compute the Friedman rank, the authors of~\cite{FernandezDelgado} impute such missing entries via the average classifier performance for this data.  We provide code for computing the Friedman rank that replicates the ranks provided in the original work of~\cite{FernandezDelgado}.

\paragraph{Average Accuracy.} Average accuracy is just the average over all available accuracies across datasets.  In this case, missing accuracies are not imputed for the average and are simply dropped. 

\paragraph{Percentage of Maximum Accuracy (PMA).}  An average over the percentage of the best classifier accuracy achieved by a given model across all datasets.  

\paragraph{P90/P95.}  An average over all datasets for which a classifier achieves within $90\%/95\%$ of the accuracy of the best model.  

\paragraph{Average Distance to Minimum (ADTM).} This metric normalizes for variance in the hardness of different datasets. Let $x_{ij}$ be the performance of method $j$ for dataset $i$, the ADTM for method $j$ is defined as $\textsf{ADTM}_j=\text{Avg}_i\round{\frac{x_{ij}-\min_j x_{ij}}{\max_j x_{ij}}}$. Note $\textsf{ADTM}\in[0,1]$, with $1$ indicating a method is the best among all methods in the collection, and $0$ indicating a method is the worst.

\section{Background on Kernel Ridge Regression}
\label{appendix: kernel ridge regression background}

We here provide a brief review of kernel ridge regression~\cite{KernelsBook}. Given a dataset $\{(x_i, y_i)\}_{i=1}^{n} \subset \mathbb{R}^{d} \times \mathbb{R}$ and a Hilbert space, $\Hilbert$, kernel ridge regression constructs an non-parametric estimator given by
\begin{align}
    \wh f_{n,\lambda}=\underset{f\in\Hilbert}{\rm argmin} \sum_{i=1}^n (f(x_i)-y_i)^2 + \lambda\norm{f}_\Hilbert^2 ;
\end{align}
where $\lambda \geq 0$ is referred to as the ridge regularization parameter.  Note this is an infinite dimensional optimization problem in a Reproducing Kernel Hilbert Space, $\Hilbert$, corresponding to a positive semi-definite kernel function $K$. By virtue of the Representer theorem~\cite{wahba1990spline}, this problem has a unique solution in the span of the data given by
\begin{align}
    \wh f_{n,\lambda}=\sum_{i=1}^n \wh\alpha_i K(x,x_i)\qquad\text{where}\quad \wh\alpha= y (K(X,X)+\lambda I_n)\inv  ;
\end{align}
where $K(X,X)_{ij} = K(x_i,x_j)$ and $y \in \mathbb{R}^{1 \times n}$.  Naively, this involves solving a $n\times n$ linear system, which can be typically solved in closed form for $n \leq 100,000$. For $n>100,000$, we apply the EigenPro solver~\cite{EigenProGPU} to approximately solve kernel regression via early-stopped, preconditioned-SGD that can run on the GPU. For $\lambda \rightarrow 0^+$, we recover the pseudo-inverse solution $\wh\alpha= y K(X,X)^\dagger $. For multi-class and multi-variate problems, $y_i$ are vector valued and we consider each class/target variable as a separate problem. 

\section{Kernel alignment and gradient outer product}
\label{appendix: Kernel Alignment}

To improve kernel selection for supervised learning, a line of research~\cite{cristianini2001kernel,cortes2012algorithms, sinha2016learning} considered selecting a kernel or a combination of kernels to maximize alignment with the following, \textit{ideal} kernel, function.

\begin{definition}
\label{def:ideal_kernel}
Suppose data $(x,f(x)) \in \Real^{d} \times \Real$ are generated by a target function $f(x)$. Then, the \textbf{ideal kernel} is $K^*(x,z) = f(x) f(z)$.  
\end{definition}

If one knows the target function $f$ beforehand, then the ideal feature map is $\psi(x) = f(x)$, as the predictor $\bm{1}^T \psi(x)$ will recover the target value exactly (assuming no label noise). Further, in the Bayesian setting, the ideal kernel averaged over the distribution of target functions will be optimal~\cite{JacotAlignment}.   We now showcase a benefit of the expected gradient outer product, $M$, by demonstrating that regression with a Mahalanobis kernel using $M$ will recover the ideal kernel when the target function is linear.

\begin{prop}
\label{prop:rfm_ideal_informal}
Let $x \in \mathbb{R}^{d}$ have density $\rho$, let $\beta \in \mathbb{R}^{d}$, and consider the linear model, i.e., $f(x) = \beta^T x$. For $z, z' \in \mathbb{R}$, let $K_M(z,z') = z^T M z'$ with $M = \mathbb{E}_{x}[\nabla f(x) \nabla f(x)^T]$.  Then, $K_M = K^*$.
\end{prop}

\begin{proof}
    Note $\nabla f(x) = \beta$ for all $x$. Hence, $M = \beta \beta^T$, and $K_M(z, z') = z^T \beta \beta^T z' = f(z) f(z') = K^*(z, z')$.
\end{proof}

Moreover, the expected gradient outer product will provably reduce the sample complexity when the target function depends on only a few relevant directions in the data, as implied by the following proposition.

\begin{prop}
\label{prop:egop_low_rank_informal}
Let $x \in \mathbb{R}^{d}$ have density $\rho$ and let the target function $f^*: \Real^d \rightarrow \Real$ be a polynomial with degree $p$ and rank $r$, i.e., $f(x) = g(Ux)$ where $U \in \Real^{r \times d}$ and $g : \Real^r \to \Real$. Let ${M = \mathbb{E}_{x}\sbrac{\nabla f^*(x) \nabla f^*(x)^T} \in \Real^{d \times d}}$. Then, there exists a fixed polynomial kernel such that kernel ridge regression on the transformed data $(M^{\frac{1}{2}}X,y)$ has the minimax sample dependence on rank, $O(r^p)$. 
\end{prop}

\begin{proof}[Proof of Proposition \ref{prop:egop_low_rank_informal}]
The gradient of the target function in directions orthogonal to the target subspace $U$ is $0$, as the function does not vary in these directions. Thus, $\nabla f^*(x)$ is in the span of $U$. Hence, for any $x' \in \Real^d$, as
\begin{align*}
    Mx' &= \mathbb{E}_{x \in X_N} \sbrac{\nabla f^*(x) \nabla f^*(x)^T x'}~,
\end{align*}
we have that $Mx'$ is also in the span of $U$. Therefore, the transformed data $M^{\frac{1}{2}} X$ lies in an $r$-dimensional subspace and has an equivalent representation in an $r$-dimensional coordinate space. Namely, for all $i,j \in [d]$, there exists $\alpha_i, \alpha_j \in \Real^r$ such that $\|x_i - x_j\| = \|\alpha_i - \alpha_j\|$. Further, the degree of the target function does not change under linear transformation or rotation. The final bound follows from the generalization error bound of linear regression for kernel ridge regression with a polynomial kernel of degree $r$. 
\end{proof}

\textbf{Remarks.} This result is in contrast to using a fixed kernel for which $\Omega(d^{p})$ samples are required to achieve better error than the trivial $0$-function by kernel ridge regression~\cite{ghorbani2021linearized}. While the above propositions assume we have knowledge of the expected gradient outer product of the target function, we note that related algorithms are optimal, even when the expected gradient outer product has not been estimated exactly. For example, kernel ridge regression using a Mahalanobis kernel with $M$ set to the neural feature matrix after 1 step of gradient descent gives the optimal dependence on the rank $r$ under certain conditions on the target function~\cite{DLSRepresentationReLU}.  We note that a related iterative procedure using kernel smoothers to simultaneously estimate a predictor and gradients achieves minimax optimality for low-rank function estimation~\cite{RecursiveMultiIndex}.

\newpage

\begin{figure}[!h]
    \centering
    \includegraphics[width=\textwidth]{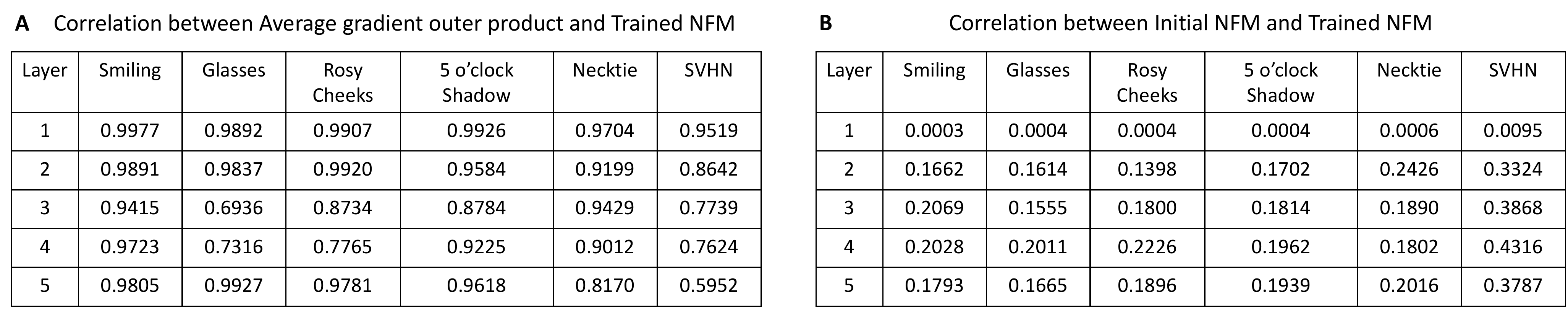}
    \caption{\textbf{(A)} The correlation between the \ajop and the trained NFM for each layer in five-hidden-layer ReLU fully connected networks trained on $6$ image classification tasks from CelebA and SVHN.  To compute the trained NFM, we subtract the layer weights at initialization from the final weights before computing the Gram matrix.  \textbf{(B)} The correlation between initial NFM and trained NFM for each layer in five-hidden-layer ReLU fully connected networks trained on $6$ image classification tasks from CelebA and SVHN.}
    \label{appendix fig: Ansatz CelebA SVHN}
\end{figure}

\newpage

\begin{figure}[!h]
    \centering
    \includegraphics[width=\textwidth]{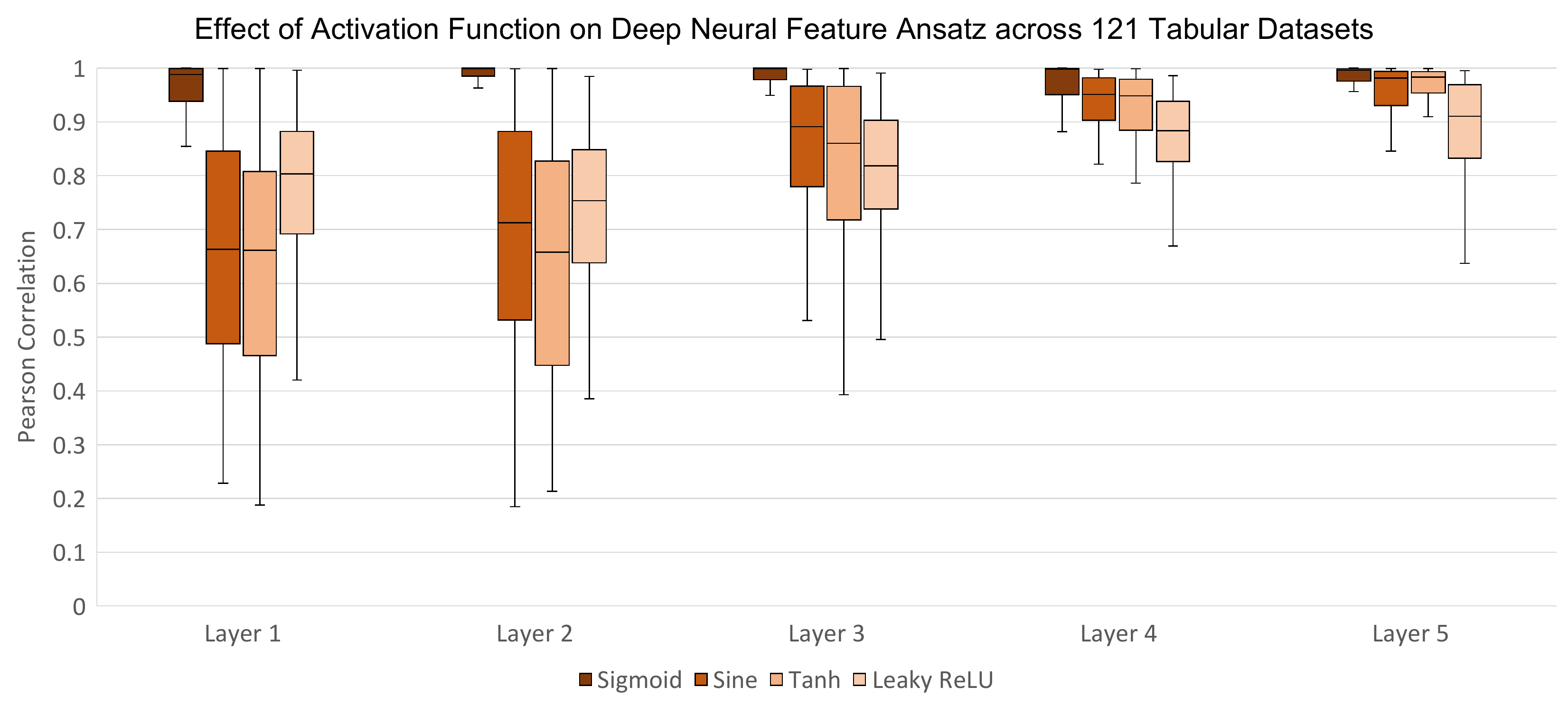}
    \caption{The correlation between the \ajop and the trained NFM for each layer in five-hidden-layer, width $64$ fully connected networks trained on $121$ tabular classification tasks using the Adam optimizer with learning rate of $10^{-4}$ and default initialization scheme from PyTorch across sigmoid, sine, hyperbolic tangent, and leaky ReLU activation.  To compute the trained NFM, we subtract the layer weights at initialization from the final weights before computing the Gram matrix.  For all activations across all layers, median correlation is above $.65$. }
    \label{appendix fig: Activation}
\end{figure}

\newpage

\begin{figure}[!h]
    \centering
    \includegraphics[width=\textwidth]{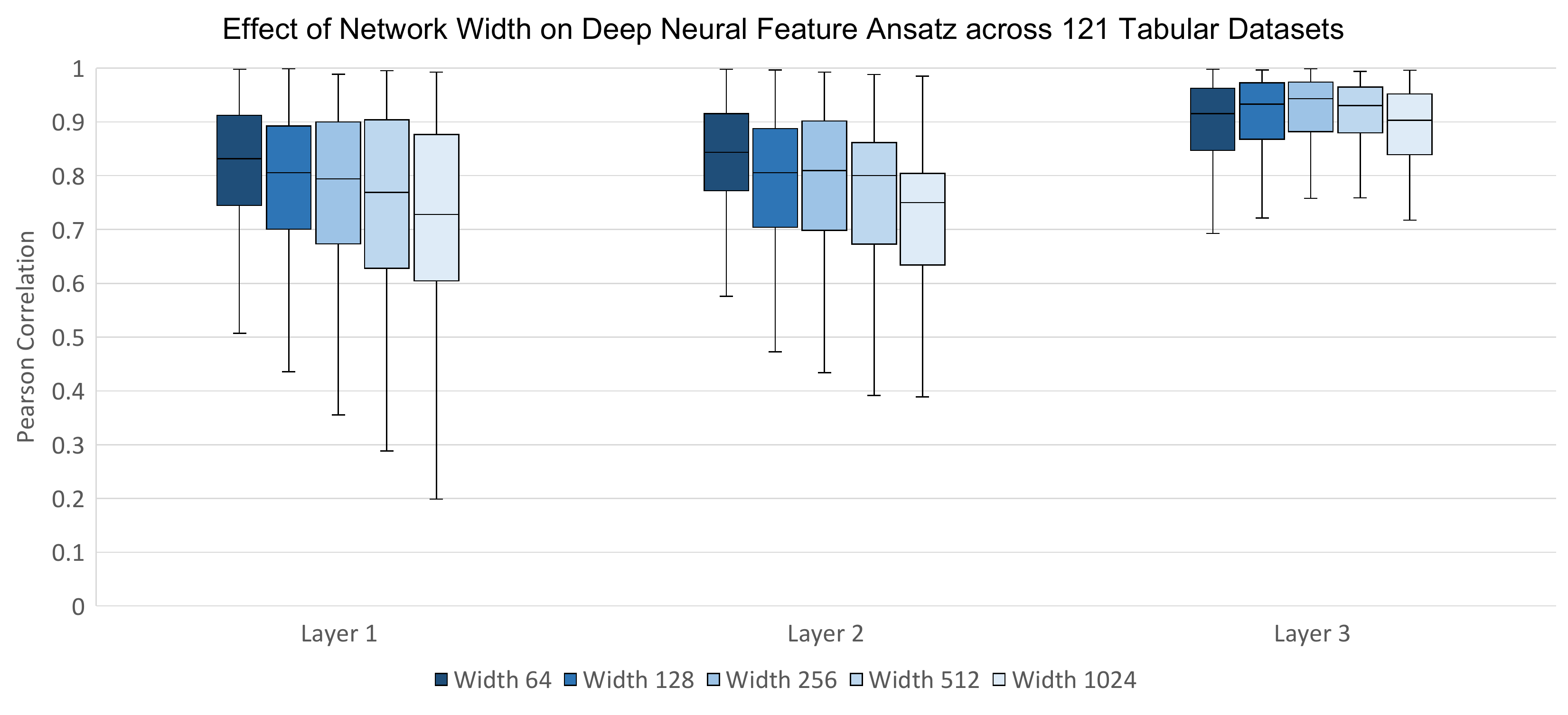}
    \caption{The correlation between the \ajop and the trained NFM for each layer in five-hidden-layer ReLU fully connected networks trained on $121$ tabular classification tasks using the Adam optimizer with learning rate of $10^{-4}$ and default initialization scheme from PyTorch across widths in the range $\{64, 128, 256, 512, 1024\}$.  We observe that lower width leads to higher correlation.  To compute the trained NFM, we subtract the layer weights at initialization from the final weights before computing the Gram matrix. }
    \label{appendix fig: Width}
\end{figure}

\newpage

\begin{figure}[!h]
    \centering
    \includegraphics[width=\textwidth]{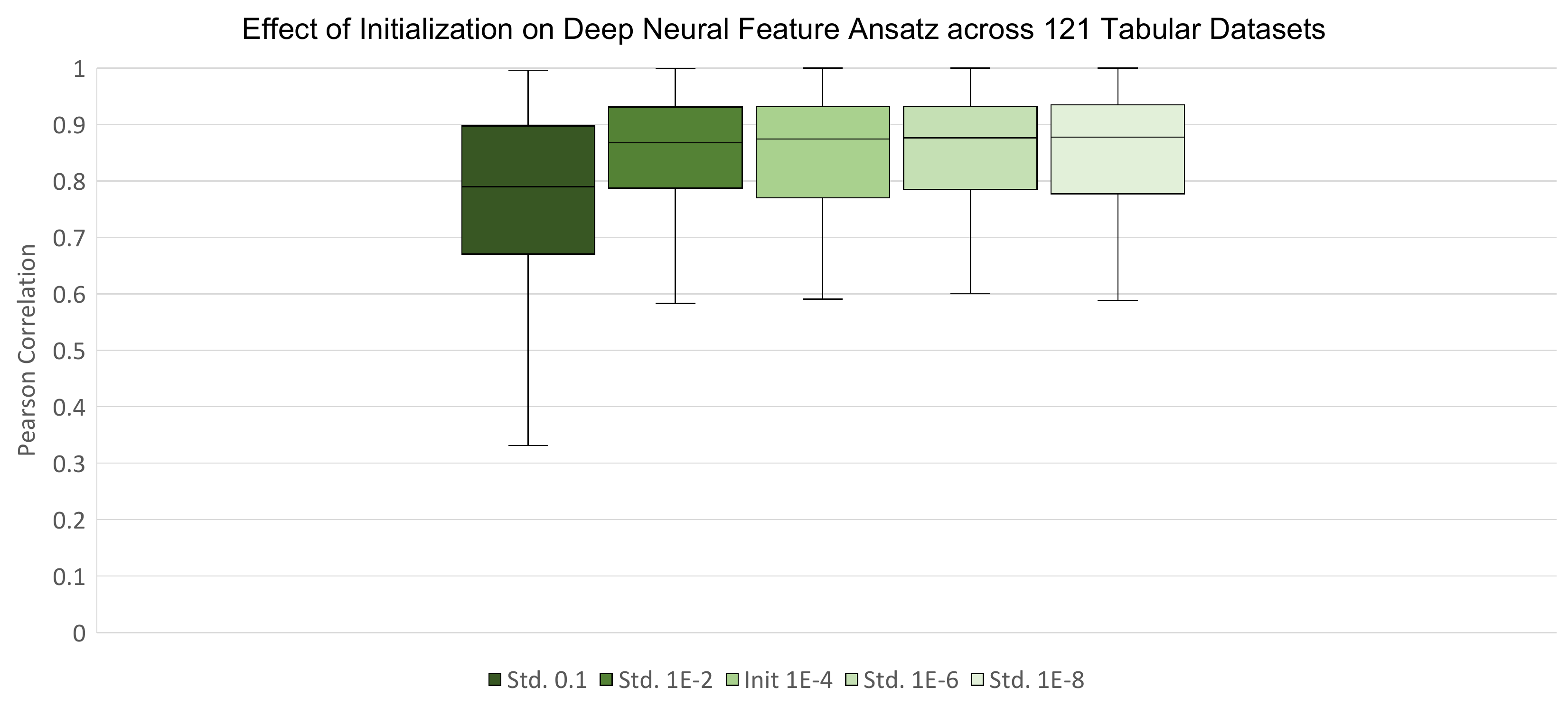}
    \caption{The correlation between the \ajop and the trained NFM for the first layer in $1$-hidden-layer, width $256$ ReLU fully connected networks trained on $121$ tabular classification tasks using the Adam optimizer with learning rate of $10^{-4}$ using an initialization scheme of Gaussian with mean $0$ and standard deviation in the range $\{10^{-1}, 10^{-2}, 10^{-4}, 10^{-6}, 10^{-8}\}$.  To compute the trained NFM, we subtract the layer weights at initialization from the final weights before computing the Gram matrix.  We observe that lower initialization scheme leads to higher correlation. }
    \label{appendix fig: Init}
\end{figure}

\newpage 

\begin{figure}[!h]
    \centering
    \includegraphics[width=\textwidth]{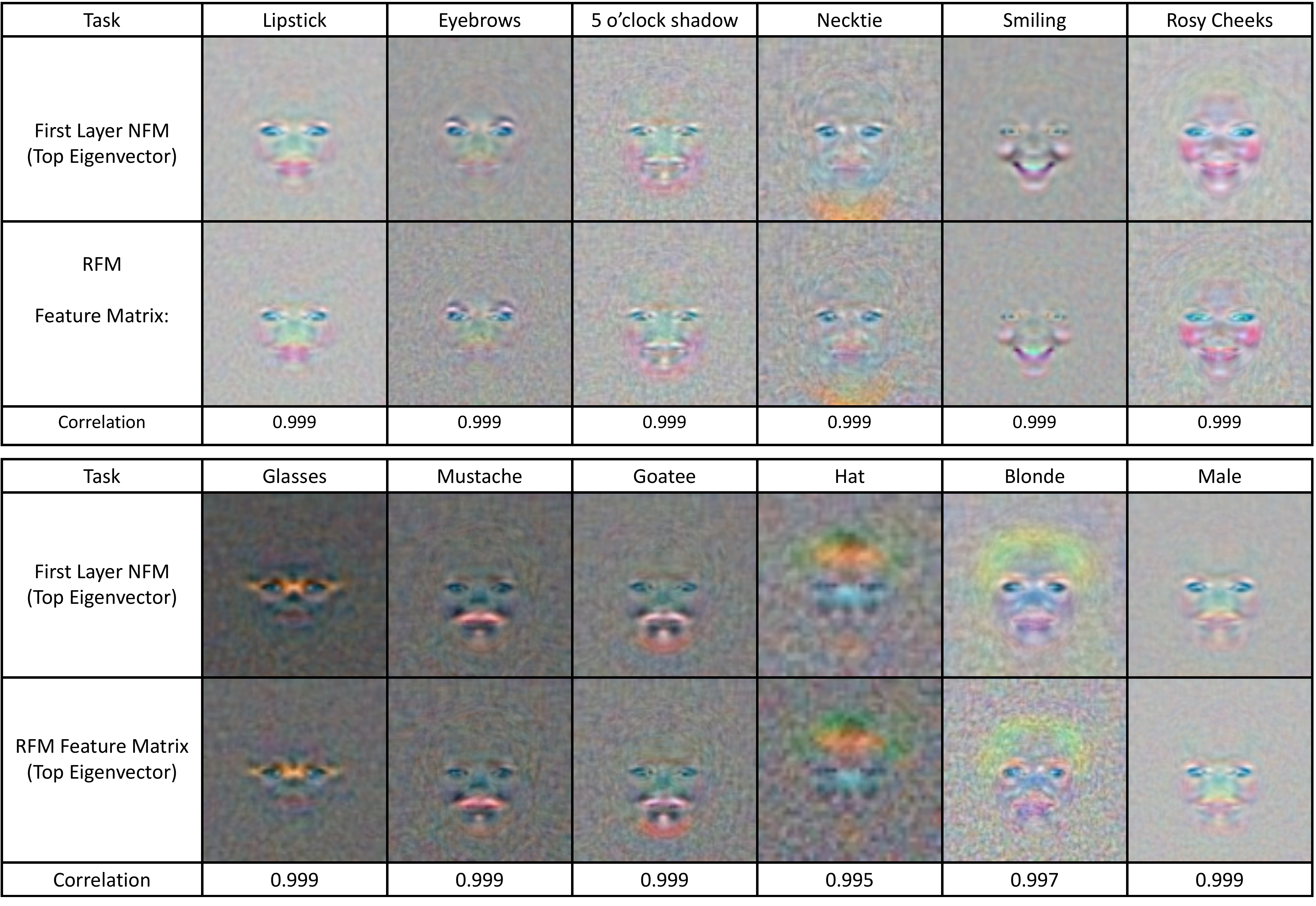}
    \caption{The top eigenvector of the first layer NFM from a two-hidden-layer, 1024 width ReLU network and from RFM feature matrices visually highlight similar features across 12 CelebA classification tasks.  These top eigenvectors are highly correlated (Pearson correlation greater than $0.99$).}
    \label{appendix fig: NN and RFM features CelebA}
\end{figure}

\newpage 

\begin{figure}[!h]
    \centering
    \includegraphics[width=\textwidth]{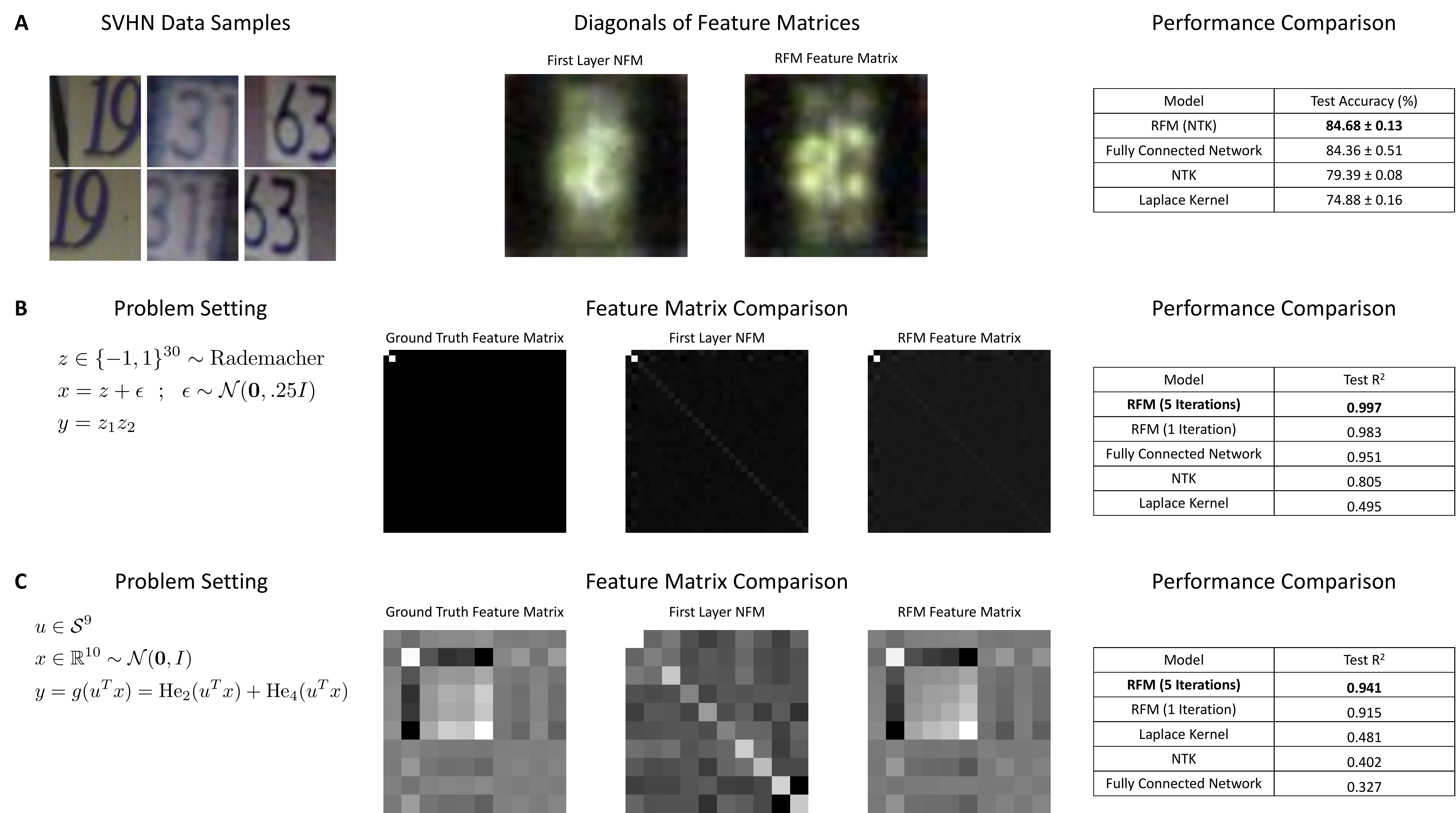}
    \caption{Comparison of first layer NFM features and RFM features and performance between two-hidden-layer, 1024 width ReLU fully connected networks and RFMs. \textbf{(A)} (Left) Samples from the SVHN dataset in which the goal is to identify the center digit from a view of potentially many digits. (Center) Upon visualizing the diagonals of the feature matrices of RFMs and deep networks, we observe that these models learn to select the center digit. (Right) By selecting the center digit, RFMs and deep networks provide a 10\% increase in test accuracy over Laplace kernels and a 5\% increase in test accuracy over NTKs.  \textbf{(B)}  (Left) We consider the low rank setup from~\cite{PreetumLimitations} in which the targets, $y$, are generated as a product of the first two coordinates of Rademacher random variables $z$.  (Center) Since we know the target function, we can compare the ground truth feature matrix against the first layer NFM of a 1 hidden layer ReLU fully connected network and the RFM feature matrix.  We observe that both models learn to select the top two coordinates.  (Right) The performance of RFMs and neural networks far exceeds of NTKs and Laplace kernels since these methods learn to select relevant coordinates for prediction. \textbf{(C)} (Left) The low rank setup from~\cite{DLSRepresentationReLU} in which the targets, $y$ are generated as a function of a projection of the input $x$ onto a 1 dimensional subspace.  Here, $u$ is on the unit sphere in $10$ dimensions and $\text{He}_2, \text{He}_4$ denote the second and fourth probabilist's Hermite polynomials. (Center) While RFMs learn to accurately approximate the ground truth gradient outer product, fully connected networks require additional training modifications, as discussed in~\cite{DLSRepresentationReLU}. (Right) As they learn the relevant subspace, RFMs far outperform 1 hidden layer ReLU fully connected networks, NTKs, and the Laplace kernel. }
    \label{appendix fig: NN and RFM features SVHN Low Rank}
\end{figure}

\begin{figure}[t]
    \centering
    \includegraphics[width=.8\textwidth]{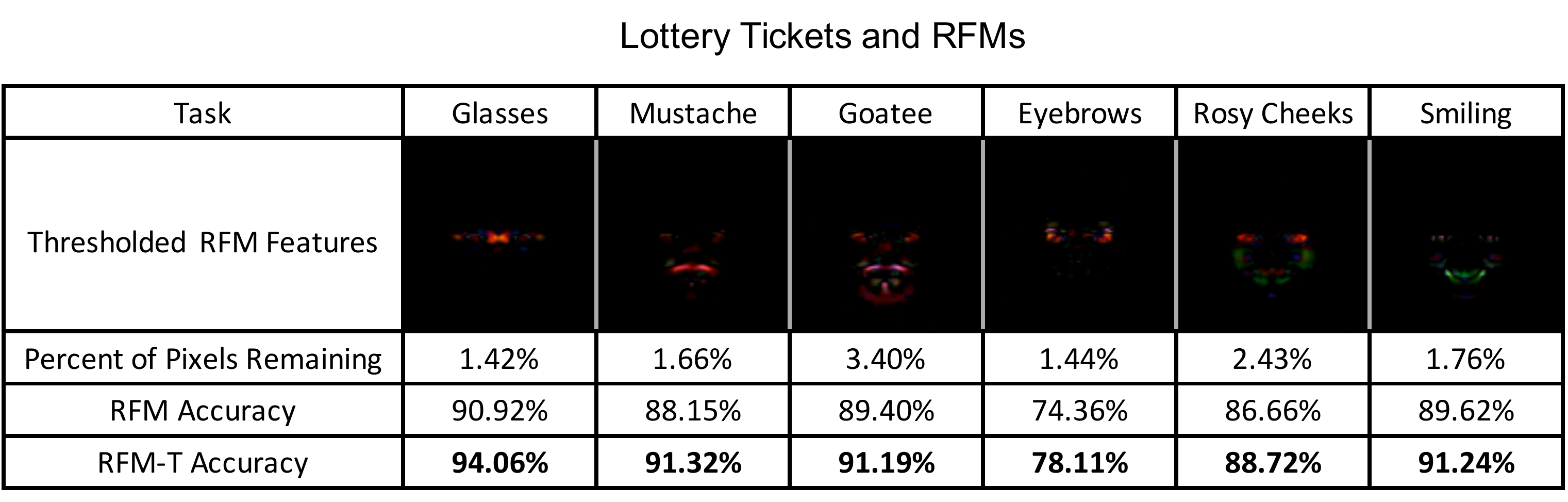}
    \caption{Connections between lottery tickets in deep networks and RFMs.  The diagonals of the feature matrices of RFMs trained on CelebA are sparse, thereby indicating that only a subset of coordinates is used for prediction.  Such sparsity suggests that we can threshold to very few pixels while still minimally affecting performance.  Indeed, re-training RFMs upon thresholding to less than $3.5\%$ of total pixels in CelebA tasks consistently improves performance for these tasks.}
    \label{appendix fig: RFM Lottery Ticket}
\end{figure}

\begin{figure}[t]
    \centering
    \includegraphics[width=.8\textwidth]{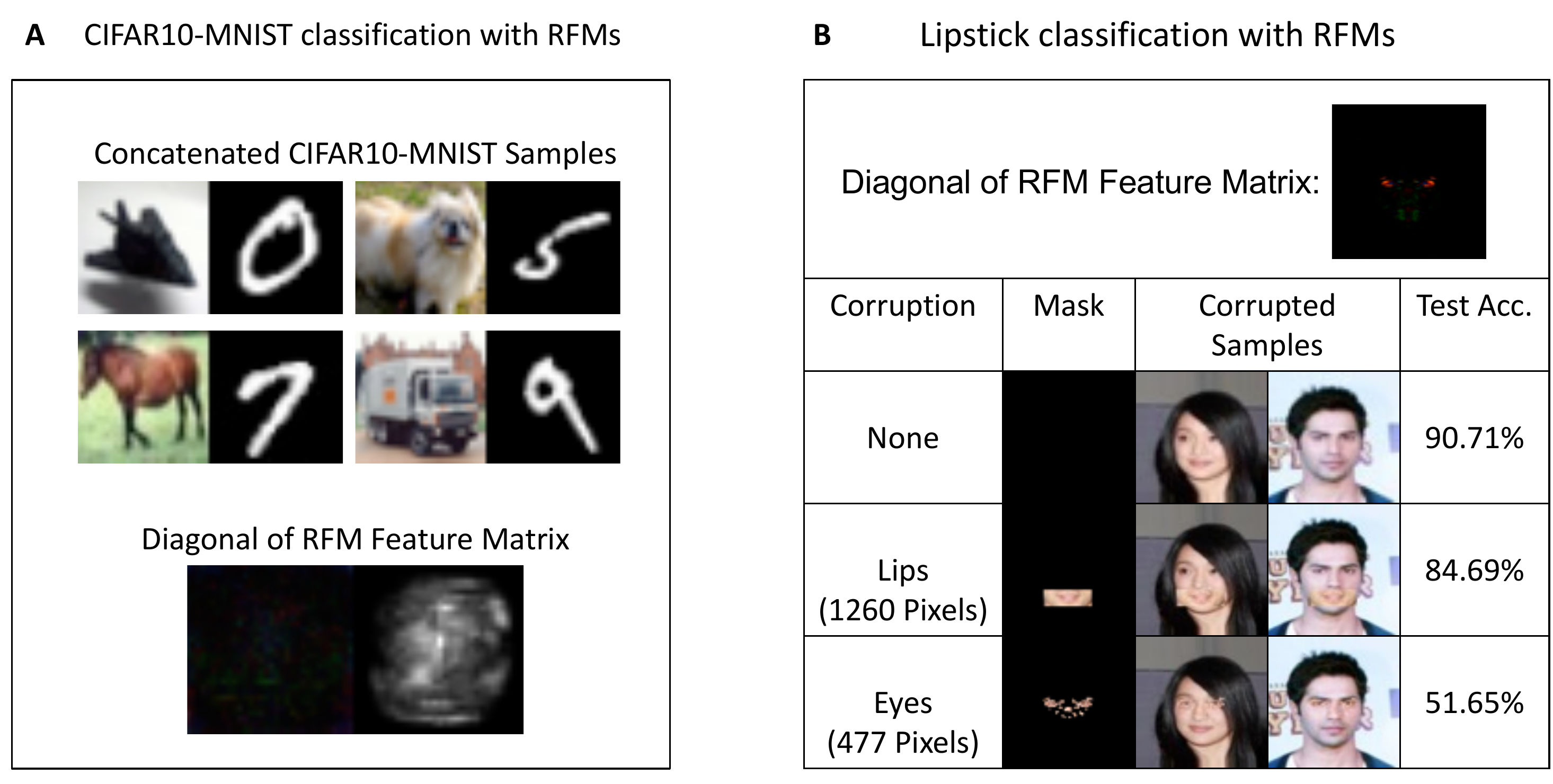}
    \caption{RFMs accurately capture simplicity biases of deep fully connected networks.  \textbf{(A)} We train RFMs on a dataset similar to that from~\cite{shah2020pitfalls} in which we concatenate images of CIFAR-10 objects with unique digits from MNIST.  Upon visualizing the diagonals of the feature matrices of RFMs, we observe that the model learns to mask the CIFAR-10 image and focus on the MNIST digit for prediction. \textbf{(B)} The diagonal of the feature matrix for an RFM trained on lipstick classification unusually indicates that eyes are used as a key feature.  We thus construct a mask based on the top RFM features and replace the eyes of all test samples with those of a single individual. The trained RFM does $39.06\%$ worse on these corrupted samples.  On the other hand, replacing the lips of all test samples with those from the same individual leads to only a minor, 6.02\%, decrease in accuracy.}
    \label{appendix fig: RFM Simplicity Bias}
\end{figure}

\newpage 

\begin{figure}[!h]
    \centering
    \includegraphics[width=1\textwidth]{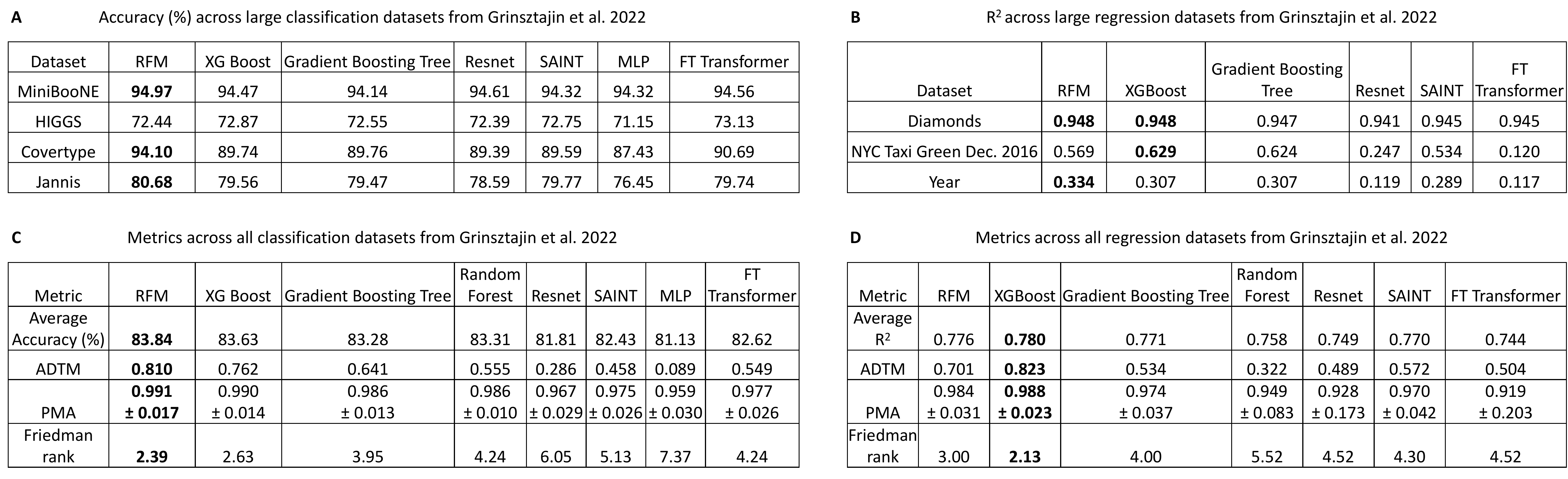}
    \caption{Performance of RFMs, XGBoost, Gradient Boosting Trees, Random Forests, ResNets, Transformers (SAINT and FT), and fully connected networks (MLPs) from~\cite{TabularDataBenchmark}. Our method used $3600$ hours in total, while all other methods used $20000$ hours for tuning, as reported in~\cite{TabularDataBenchmark}. Results from all models other than RFMs are reported from the tables provided by~\cite{TabularDataBenchmark}. All metrics and training details are outlined in Methods.  Large tasks have $50000$ training examples except for Jannis ($40306$ examples) and Diamonds ($37758$ examples). The medium tasks have at most $10000$ samples. We show (\textbf{A}) average accuracy on large classification tasks and (\textbf{B}) average $R^2$ on large regression tasks. We compare model performance through commonly used metrics across all datasets for (\textbf{C}) classification, and (\textbf{D}) regression.}
    \label{fig: Transformers}
\end{figure}

\newpage 

\begin{table}
\begin{adjustbox}{width=\textwidth}
\begin{filecontents*}{Tables/reg_no_categorical.txt}
n-train,data-keyword,FT Transformer,GradientBoostingTree,RandomForest,Resnet,SAINT,XGBoost,ours,best method
4547.0,wine-quality,-,0.458,0.504,-,-,0.498,0.497,RandomForest
5457.0,isolet,-,-,-,-,-,-,0.868,ours
5734.0,cpu-act,-,0.985,0.983,-,-,0.986,0.986,XGBoost
7056.0,sulfur,-,0.806,0.859,-,-,0.865,0.913,ours
7484.0,Brazilian-houses,-,0.996,0.993,-,-,0.998,0.938,XGBoost
9625.0,Ailerons,-,0.843,0.839,-,-,0.844,0.841,XGBoost
9752.0,MiamiHousing2016,-,0.924,0.924,-,-,0.936,0.934,XGBoost
10000.0,Bike-Sharing-Demand,-,0.69,0.687,-,-,0.695,0.689,XGBoost
10000.0,california,-,0.846,0.83,-,-,0.853,0.867,ours
10000.0,diamonds,-,0.945,0.945,-,-,0.946,0.945,XGBoost
10000.0,elevators,-,0.863,0.841,-,-,0.908,0.923,ours
10000.0,fifa,-,0.663,0.655,-,-,0.668,0.653,XGBoost
10000.0,house-16H,-,0.541,0.486,-,-,0.548,0.489,XGBoost
10000.0,house-sales,-,0.884,0.871,-,-,0.887,0.883,XGBoost
10000.0,houses,-,0.84,0.829,-,-,0.852,0.864,ours
10000.0,medical-charges,-,0.979,0.979,-,-,0.979,0.979,GradientBoostingTree
10000.0,nyc-taxi-green-dec-2016,-,0.554,0.563,-,-,0.553,0.532,RandomForest
10000.0,pol,-,0.99,0.989,-,-,0.99,0.991,ours
10000.0,superconduct,-,0.905,0.909,-,-,0.911,0.911,ours
10000.0,year,-,0.271,0.241,-,-,0.282,0.303,ours
37758.0,diamonds,0.945,0.947,-,0.941,0.945,0.948,0.948,ours
50000.0,nyc-taxi-green-dec-2016,0.12,0.624,-,0.247,0.534,0.629,0.569,XGBoost
50000.0,year,0.117,0.307,-,0.119,0.289,0.307,0.334,ours
\end{filecontents*}
\csvautotabular{Tables/reg_no_categorical.txt}
\end{adjustbox}
\caption{Regression $R^2$ without categorical variables.}
\label{appendix: table reg without cat}
\end{table}

\begin{table}
\begin{adjustbox}{width=\textwidth}
\begin{filecontents*}{Tables/clf_no_categorical.txt}
n-train,data-keyword,FT Transformer,GradientBoostingTree,MLP,RandomForest,Resnet,SAINT,XGBoost,ours,best method
1787.0,wine,77.54,78.77,77.62,78.96,77.95,77.09,79.85,80.67,ours
2220.0,phoneme,85.28,86.78,84.95,88.55,86.21,85.58,86.48,88.16,RandomForest
3631.0,kdd-ipums-la-97-small,88.96,88.38,87.95,87.95,88.07,89.02,88.26,88.62,SAINT
5325.0,eye-movements,58.62,63.75,56.89,65.04,57.41,58.93,65.54,61.14,XGBoost
7057.0,pol,98.47,97.94,94.27,98.21,94.81,98.14,98.05,98.33,FT Transformer
7404.0,bank-marketing,80.42,80.27,79.18,79.82,79.37,79.09,80.44,79.73,XGBoost
9363.0,MagicTelescope,85.09,85.88,84.7,85.6,85.78,85.1,85.92,86.5,ours
9441.0,house-16H,88.16,88.2,87.78,87.8,87.5,88.16,88.83,87.78,XGBoost
10000.0,Higgs,70.62,71.08,68.86,70.76,69.44,70.72,71.42,70.73,XGBoost
10000.0,MiniBooNE,93.74,93.19,93.54,92.65,93.68,93.52,93.62,93.93,ours
10000.0,california,88.62,89.82,86.0,89.21,87.57,89.07,90.17,90.29,ours
10000.0,covertype,81.32,81.87,78.89,82.73,80.26,80.31,81.9,85.95,ours
10000.0,credit,76.53,77.22,75.99,77.28,76.1,75.99,77.38,77.66,ours
10000.0,electricity,82.0,86.16,81.04,86.14,80.86,81.77,86.83,82.93,XGBoost
10000.0,jannis,76.55,77.02,74.57,77.27,74.6,77.26,77.78,78.28,ours
40306.0,jannis,79.74,79.47,76.45,78.85,78.59,79.77,79.56,80.68,ours
50000.0,Higgs,73.13,72.55,71.15,71.98,72.39,72.75,72.83,72.44,FT Transformer
50000.0,MiniBooNE,94.34,94.14,94.32,93.53,94.44,94.32,94.43,94.97,ours
50000.0,covertype,90.69,89.76,87.43,90.59,89.39,89.53,89.74,94.1,ours
\end{filecontents*}
\csvautotabular{Tables/clf_no_categorical.txt}
\end{adjustbox}
\caption{Classification accuracy without categorical variables.}
\label{appendix: table class without cat}
\end{table}

\begin{table}
\begin{adjustbox}{width=\textwidth}
\begin{filecontents*}{Tables/reg_yes_categorical.txt}
n-train,data-keyword,FT Transformer,GradientBoostingTree,HistGradientBoostingTree,RandomForest,Resnet,XGBoost,ours,best method
2836.0,analcatdata-supreme,0.977,0.981,0.982,0.981,0.978,0.983,0.987,ours
2946.0,Mercedes-Benz-Greener-Manufacturing,0.548,0.578,0.576,0.575,0.545,0.578,0.575,GradientBoostingTree
6048.0,visualizing-soil,0.998,1.0,1.0,1.0,0.998,1.0,1.0,RandomForest
6219.0,yprop-4-1,0.037,0.056,0.063,0.095,0.021,0.08,0.072,RandomForest
7484.0,Brazilian-houses,0.883,0.995,0.993,0.993,0.878,0.998,0.906,XGBoost
9999.0,OnlineNewsPopularity,0.143,0.153,0.156,0.149,0.13,0.162,0.139,XGBoost
10000.0,Bike-Sharing-Demand,0.937,0.942,0.942,0.938,0.936,0.946,0.933,XGBoost
10000.0,SGEMM-GPU-kernel-performance,1.0,1.0,1.0,1.0,1.0,1.0,1.0,ours
10000.0,black-friday,0.379,0.615,0.616,0.609,0.36,0.619,0.612,XGBoost
10000.0,diamonds,0.99,0.99,0.991,0.988,0.989,0.991,0.991,XGBoost
10000.0,house-sales,0.891,0.891,0.89,0.875,0.881,0.896,0.891,XGBoost
10000.0,nyc-taxi-green-dec-2016,0.511,0.573,0.539,0.585,0.451,0.578,0.549,RandomForest
10000.0,particulate-matter-ukair-2017,0.673,0.683,0.69,0.674,0.658,0.691,0.662,XGBoost
37758.0,diamonds,-,0.992,0.993,-,-,0.993,0.993,ours
50000.0,SGEMM-GPU-kernel-performance,-,1.0,1.0,-,-,1.0,1.0,XGBoost
50000.0,black-friday,-,0.631,0.636,-,-,0.639,0.623,XGBoost
50000.0,nyc-taxi-green-dec-2016,-,0.636,0.585,-,-,0.648,0.589,XGBoost
50000.0,particulate-matter-ukair-2017,-,0.706,0.71,-,-,0.712,0.69,XGBoost
\end{filecontents*}
\csvautotabular{Tables/reg_yes_categorical.txt}
\end{adjustbox}
\caption{Regression $R^2$ with categorical variables.}
\label{appendix: table reg with cat}
\end{table}

\begin{table}
\begin{adjustbox}{width=\textwidth}
\begin{filecontents*}{Tables/clf_yes_categorical.txt}
n-train,data-keyword,FT Transformer,GradientBoostingTree,HistGradientBoostingTree,RandomForest,Resnet,SAINT,XGBoost,ours,best method
3479.0,rl,70.31,77.62,76.05,79.79,70.56,68.2,77.01,70.47,RandomForest
3522.0,KDDCup09-upselling,78.05,-,-,-,76.98,77.8,-,-,FT Transformer
3589.0,KDDCup09-upselling,-,80.36,80.61,80.02,-,-,79.56,77.24,HistGradientBoostingTree
5325.0,eye-movements,59.83,63.94,63.58,65.73,57.93,58.54,66.77,62.75,XGBoost
10000.0,compass,75.34,74.09,75.15,79.28,74.46,71.87,76.91,77.5,RandomForest
10000.0,covertype,86.74,85.56,84.47,85.89,85.27,84.95,86.42,89.5,ours
10000.0,electricity,84.16,87.97,88.15,87.76,82.64,83.35,88.69,85.33,XGBoost
10000.0,road-safety,76.74,76.21,76.45,75.88,76.08,76.43,76.69,75.48,FT Transformer
50000.0,covertype,93.61,93.22,92.09,93.35,92.24,92.58,93.31,95.58,ours
50000.0,road-safety,78.95,78.73,78.85,78.13,78.41,77.96,80.29,77.56,XGBoost
\end{filecontents*}
\csvautotabular{Tables/clf_yes_categorical.txt}
\end{adjustbox}
\caption{Classification accuracy with categorical variables.}
\label{appendix: table class with cat}
\end{table}

\end{document}